\documentclass[10pt,twocolumn,letterpaper]{article}

\usepackage{cvpr}              
\usepackage{float}
\usepackage{listings}
\usepackage{tcolorbox}
\usepackage[toc,page,header]{appendix}
\usepackage{minitoc}
\usepackage[T1]{fontenc}
\usepackage{afterpage}

\definecolor{cvprblue}{rgb}{0.21,0.49,0.74}
\usepackage[pagebackref,breaklinks,colorlinks,allcolors=cvprblue]{hyperref}

\usepackage{algorithm}
\usepackage{algorithmicx}

\usepackage{amssymb}		
\usepackage{graphicx}		
\usepackage{amsmath}
\usepackage{amsthm}
\usepackage{algpseudocode}
\usepackage{comment}
\usepackage{dsfont}

\newtheorem{theorem}{Theorem}
\newtheorem{lemma}{Lemma}
\newtheorem{definition}{Definition}
\newtheorem{corollary}{Corollary}

\newtheorem{proposition}{Proposition}
\newtheorem{claim}{Claim}
\newtheorem{remark}{Remark}

\newtheorem{assumption}{Assumption}

\DeclareMathOperator*{\argmin}{arg\,min}

\newcommand{\Tr}{\operatorname{Tr}}

\def\paperID{18588} 
\def\confName{CVPR}
\def\confYear{2025}

\title{Where Did Your Model Learn That? Label-free \\Influence for Self-supervised Learning}

\author{
  \hspace{5mm}Nidhin Harilal$^{1}$\thanks{Equal contribution}\footnotemark[1] \hspace{5mm}  
  \and
  \hspace{5mm}Amit Kiran Rege$^{1}$\footnotemark[1] \hspace{5mm}
  \and
  \hspace{5mm}Reza Akbarian Bafghi$^{1}$  \hspace{5mm}\vspace{0.5mm} 
  \and
  Maziar Raissi$^{2}$ 
  \and
Claire Monteleoni$^{1,3}$\vspace{1mm} 
\and
  $^{1}$University of Colorado Boulder,
  $^{2}$University of California, Riverside, $^{3}$INRIA Paris
  \and 
  {\tt\small
\{nidhin.harilal, amit.rege, reza.akbarianbafghi, cmontel\}@colorado.edu 
}\\{\tt\small maziar.raissi@ucr.edu}
}

\begin{document}
\maketitle
\begin{abstract}
Self-supervised learning (SSL) has revolutionized learning from large-scale unlabeled datasets, yet the intrinsic relationship between pretraining data and the learned representations remains poorly understood. Traditional supervised learning benefits from gradient-based data attribution tools like influence functions that measure the contribution of an individual data point to model predictions. However, existing definitions of influence rely on labels, making them unsuitable for SSL settings. We address this gap by introducing Influence-SSL, a novel and label-free approach for defining influence functions tailored to SSL. Our method harnesses the stability of learned representations against data augmentations to identify training examples that help explain model predictions. We provide both theoretical foundations and empirical evidence to show the utility of Influence-SSL in analyzing pre-trained SSL models. Our analysis reveals notable differences in how SSL models respond to influential data compared to supervised models. Finally, we validate the effectiveness of Influence-SSL through applications in duplicate detection, outlier identification and fairness analysis. Code is available at: \url{https://github.com/cryptonymous9/Influence-SSL}.
\end{abstract}

\section{Introduction}


\begin{figure}
    \centering
    \includegraphics[width=\linewidth]{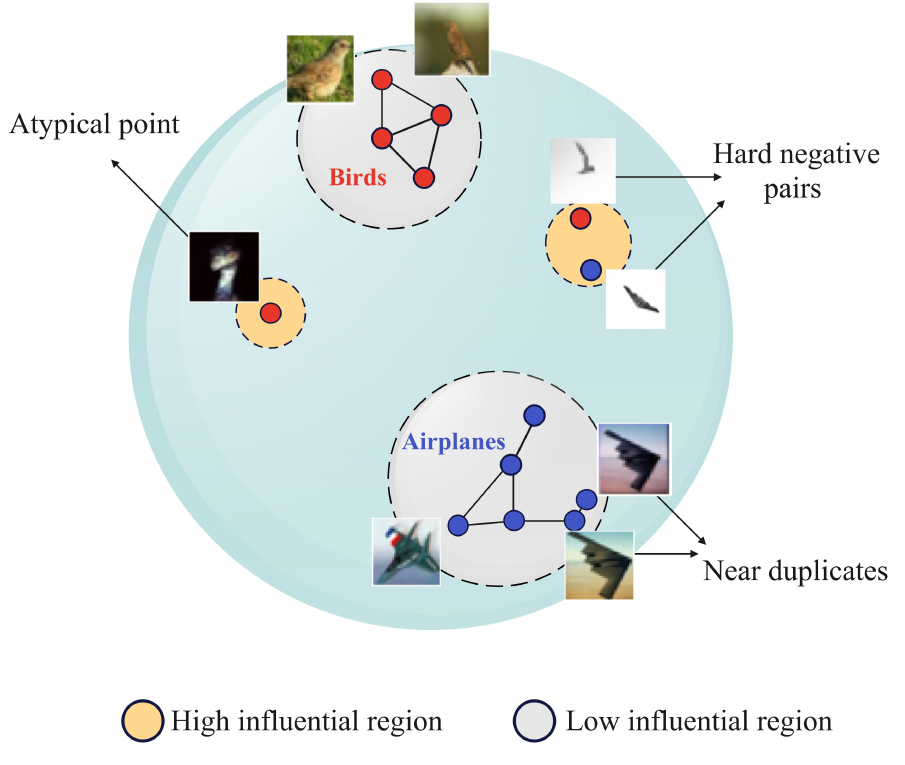}
    \caption{Illustration of scenarios for varying influence levels for self-supervised learning: (1) Hard negative examples, which challenge the model's decision boundary, should have high influence. (2) Atypical examples, representing rare or outlier data, should also exhibit high influence. (3) Near-duplicate examples should maintain low influence due to their redundant nature and low contribution to learning.}
    \label{fig:intro}
\end{figure}

\textit{Influence functions}~\cite{koh2017understanding} have proven valuable in supervised learning for assessing the impact of individual training examples on model behavior, providing insights into data memorization and learning dynamics~\cite{feldman2020does, bae2022if, grosse2023studying, zhang2022rethinking}. However, extending this concept to SSL presents unique challenges, as existing definitions rely heavily on labeled data and often require computationally expensive model retraining~\cite{feldman2020does, guu2023simfluence}.  Unlike supervised learning, where labels provide a direct link between training examples and model behavior, SSL tasks often involve pretext objectives that obscure how specific data points influence model training and performance. This makes extending influence estimation to SSL particularly challenging. This demands the development of a new theoretical framework for influence estimation tailored to SSL, one that can account for the distinct characteristics of SSL learning objective.

The importance of understanding data influence in SSL extends beyond theoretical interest; it has practical implications for data curation, robustness analysis, and debiasing~\cite{bae2022if}. By quantifying the influence of training instances, practitioners can identify harmful or overly influential examples, leading to more stable and interpretable learning processes. Yet, the absence of labels and the complexity of pretext tasks mean traditional influence estimation techniques, which depend on gradient-based methods or retraining schemes~\cite{koh2017understanding}, are not directly applicable~\cite{harilal2024influence}.


We introduce \textit{Influence-SSL}, a novel framework for analyzing how training examples shape self-supervised representations. Our approach reformulates influence functions for the SSL setting by leveraging the fundamental invariance-distinctiveness trade-off inherent in SSL objectives. Rather than relying on labels, we quantify influence through representation stability against data augmentations, providing a natural measure aligned with SSL training objectives. We provide some intuition on influential examples computed using Influence-SSL in Figure~\ref{fig:intro}.

We provide both a theoretical angle and empirical validation for understanding SSL dynamics using influence analysis. We show that Influence-SSL holds several mathematical properties required from influence estimation in a simplified setting. Additionally, we establish the connection between influence scores, representation structure, and augmentation consistency. Comprehensive experiments across various SSL frameworks confirm that our influence measures offer consistent and interpretable insights into model behavior. These insights enable practical applications such as duplicate detection, outlier identification, and fairness analysis. Our results indicate that influence analysis can uncover significant patterns in how SSL methods process training data, with potential benefits for enhancing model robustness and transparency.

\section{Background}

\subsection{Classical Influence Functions}

In this section, we briefly review influence functions~\cite{koh2017understanding, bae2022if}, which are a classical technique from robust statistics~\cite{hampel1974influence} that was recently introduced to deep learning by~\cite{koh2017understanding} to help analyze how individual training examples affect a model's learned parameters and predictions.

Consider a dataset $S = \{z_i\}_{i=1}^n$ where each $z_i$ may represent either a single sequence (for self-supervised learning) or an input-target pair $z_i = (x_i, y_i)$ (for supervised learning). The model parameters $\theta \in \mathbb{R}^d$ are learned through empirical risk minimization of a loss function $\mathcal{L}$:
\begin{equation}
\theta^* = \argmin_{\theta \in \mathbb{R}^d} \frac{1}{n}\sum_{i=1}^n \mathcal{L}(z_i, \theta)
\end{equation}

To understand the effect of removing a new training example $z_m$ to the dataset, we can parameterize the training objective by the weight $\epsilon \in \mathbb{R}$ given to this example. This yields the response function:
\begin{align*}
\theta^*(\epsilon) &= \argmin_{\theta \in \mathbb{R}^d} J(\theta, S_{-\epsilon}) \\
&= \argmin_{\theta \in \mathbb{R}^d} \frac{1}{n}\sum_{i=1}^n \mathcal{L}(z_i, \theta) - \epsilon\mathcal{L}(z_m, \theta)
\end{align*}

The influence of $z_m$ on $\theta^*$ is defined as the first-order Taylor approximation to the response function at $\epsilon = 0$. Under appropriate regularity conditions, the Implicit Function Theorem yields:
\begin{equation}
\mathcal{I}_{\theta^*}(z_m) = \frac{d\theta^*}{d\epsilon}\bigg|{\epsilon=0} = -H^{-1}\nabla_\theta\mathcal{L}(z_m, \theta^*)
\end{equation}
where $H = \nabla^2_\theta J(\theta^*, S)$ is the Hessian of the loss. This allows us to linearly approximate the change in parameters when adding $z_m$ with weight $\epsilon = \frac{1}{n}$:
\begin{equation}
\theta^*(\epsilon) - \theta^* \approx I_{\theta^*}(z_m)\epsilon = -H^{-1}\nabla_\theta\mathcal{L}(z_m, \theta^*)\epsilon
\end{equation}


Since the direct parameter influence $\mathcal{I}_{\theta^*}$ can be difficult to interpret, it is common to instead compute the influence on a measurable quantity $g(\theta)$, such as the validation loss or model predictions for a query point $z_q$. By the chain rule:
\begin{align}
\mathcal{I}_f(z_m) &= \nabla_\theta g(\theta^*)^\top \mathcal{I}_{\theta^*}(z_m) \nonumber \\
        &= -\nabla_\theta g(\theta^*)^\top H^{-1}\nabla_\theta\mathcal{L}(z_m, \theta^*) \label{eq:inf}
\end{align}

Thus, one can efficiently analyze how an individual training examples affects both model parameters and downstream quantities of interest, without requiring explicit model retraining.

In this paper, we are interested in the case where $g = \mathcal{L}(z_m, \theta^*)$ i.e. the influence of a point on itself, also referred to sometimes as self-influence~\cite{feldmanzhang}. 

\subsection{Self Supervised Learning}
In SSL, we train a model using only input data without explicit labels. Let $x \in \mathcal{X}$ be an input, and $\tau: \mathcal{X} \rightarrow \mathcal{X} \times \mathcal{X}$ be a stochastic transformation function that generates a pair of views $(x_a, x_b) = \tau(x)$. The model consists of an encoder $f_\theta: \mathcal{X} \rightarrow \mathbb{R}^d$ that maps inputs to embeddings, where $\theta \in \mathbb{R}^D$ represents the model parameters.

Given a dataset $S = \{x_i\}_{i=1}^n$, the model is trained to maximize the similarity between embeddings of different views of the same input while minimizing similarity between embeddings of different inputs. 


The optimal parameters are found through empirical risk minimization:
\begin{equation}
\theta^* = \argmin_{\theta \in \mathbb{R}^D} \frac{1}{n}\sum_{i=1}^n \mathcal{L}(x_i, \theta)
\end{equation}

Our setup enables the analysis of influence functions in the self-supervised setting by treating each $x_i$ as a training example $z_i$, without requiring modification to the influence computation methods described above.

\subsection{Related Work}
Data Attribution (DA) techniques explain model predictions by analyzing the training data that shaped the model~\cite{hammoudeh23}. DA methods are broadly categorized into retraining-based and gradient-based approaches. Retraining-based methods, including leave-one-out~\cite{cook82, feldmanzhang}, Shapley value techniques~\cite{shapley:book1952, ghorbani19}, and Datamodels~\cite{ilyas22a}, assess data impact by retraining the model on subsets, which is computationally expensive and limits scalability. Gradient-based methods estimate data influence via parameter sensitivity. Notable examples are representer point methods~\cite{yeh18}, TracIn~\cite{pruthi20}, and influence functions~\cite{koh2017understanding}, the focus of this work. Extensions of influence functions explore group data effects~\cite{koh19group}, higher-order information~\cite{basu20second}, and normalized rankings~\cite{barshan20a}.

The fundamental idea of quantifying influence based on the impact of individual training data points on model predictions has been extensively studied in supervised learning~\cite{arpit2017closer, zhang2021understanding, carlini2019secret}. For instance, Feldman et al.~\cite{feldman2020does} study the effect of long-tailed distributions on generalization via influence functions. However, such studies have been limited in self-supervised learning (SSL). While most influence based DA analysis has been limited to the small neural networks, a recent study~\cite{grosse2023studying} extends influence based DA analysis to large-language models, with the main contribution being scalability. Despite plenty of literature on influence functions, they have not been studied for SSL settings. Our work focuses on the SSL setting for visual data and provides theoretical analysis and justifications for our results. 


\section{Influence-SSL}
\label{sec:inf}
Extending influence functions (Equation~\ref{eq:inf}), traditionally applied in supervised learning, to self-supervised learning (SSL) requires selecting a label-free objective that reveals representation properties in SSL. Specifically, we ask: What choice of $g$ in Equation~\ref{eq:inf} would help us in understanding the impact of a training point on representation? To this end, we propose a novel adaptation of influence functions for SSL, leveraging the invariance-distinctiveness trade-off inherent in SSL objectives. Our method estimates a sample's influence by evaluating how its exclusion affects the model’s ability to align augmented views.

We formalize this concept by introducing a refined definition of the influence score $\mathcal{I}$ for an unlabelled image $x_i$, utilizing a pre-trained SSL model $f_{\theta}$:
\begin{equation}\label{eq:inf-ssl}
\mathcal{I}(f, i) = -\nabla_\theta \mathcal{L}(f_\theta(x_i), f_\theta(\hat{x}_i))^\top H_\theta^{-1} \nabla_\theta \mathcal{L}(f_\theta(x_i), f_\theta(\hat{x}_i))
\end{equation}
In this formulation, $\mathcal{L}$ represents the cosine distance between an image $x_i$ and its augmented variant $\hat{x}_i$, expressed as $\mathcal{L}(t_i, t_j) = 1-\frac{t_i \cdot t_j}{||t_i|| ||t_j||}$ while $H_\theta$ represents the Hessian of the model's loss over the dataset with respect to it's parameters. The computation of $H_\theta^{-1}$ is efficiently implemented through LoGra~\cite{choe2024dataworthgptllmscale}.
 
The choice of cosine distance as our objective $\mathcal{L}$ is deliberate, given the range of loss functions $\tilde{\mathcal{L}}$ across SSL methods. Both theoretical and empirical justifications support this selection: methods like Barlow Twins~\cite{zbontar2021barlow}, and BYOL~\cite{byol} directly optimize cosine similarity, while theoretical studies~\cite{huang2021towards, wang2022chaos} have shown its implicit maximization in contrastive frameworks, and recent work~\cite{zhang2022mask} reveals similar alignment in Masked AutoEncoders through mask-induced positive pairs.

We note that our proposed definition quantifies how a training point affects its own representation consistency, making it technically a measure of SSL $\emph{self}$-influence. However, since our primary focus is on analyzing how individual training examples impact their own learned representations, we use the terms Influence-SSL and SSL self-influence interchangeably throughout this work. While our definition naturally extends to measuring influence on separate test points analogously to the supervised setting, we defer a comprehensive investigation of such cross-point influences to future work.

\subsection{Understanding Influence-SSL in a Simplified Setting}
To provide theoretical justification for our influence definition, we analyze a simplified setting with linear networks, which have been studied extensively in the deep learning theory literature~\cite{misiakiewicz2023lectureslinearizedneuralnetworks}, and small perturbation-based augmentations. This analysis reveals that our proposed influence measure naturally emerges from fundamental SSL principles, while highlighting key differences from traditional supervised influence functions.

\begin{theorem}[Influence-SSL Characterization]\label{thm:inf-ssl}
Consider a two-layer linear network $f(x) = v^T(Wx)$ with parameters $W \in \mathbb{R}^{k \times d}$ and $v \in \mathbb{R}^k$, and an augmentation function $x_{aug} = x + \varepsilon\delta(x)$ where $|\delta(x)| = 1$, and $\varepsilon \ll 1$. Under the squared Euclidean distance loss $L_{ssl}(W;x) = |Wx - Wx_{aug}|^2$, the influence of a training point $x$ with regularization parameter $\lambda \geq 0$ is given by:
\[I_{ssl}^\lambda(x) = -4\varepsilon^4|W\delta(x)|^2\frac{1}{\lambda + 2\varepsilon^2}\]
Moreover, as $\lambda \to 0$, the influence simplifies to:
\[I_{ssl}(x) = -2\varepsilon^2|W\delta(x)|^2\]
\end{theorem}

The proof proceeds by deriving the supervised gradients and Hessian, transitioning to the SSL setting, and finally computing the Influence-SSL  through careful analysis of the regularized Hessian inverse. The complete proof is provided in Appendix~\ref{suppl:theory}.

This result reveals that in a simplified SSL setting, the influence of a training point is proportional to $|W\delta(x)|^2$, where $\delta(x)$ captures input-dependent augmentation. While classical influence functions measure effects on prediction loss, our Influence-SSL quantifies representation alignment between augmented views, naturally aligning with SSL's core objective of learning invariant representations.

Influence magnitude reveals training dynamics:  low-influence points ($|W\delta(x)|^2$ small) indicate examples where representations are already invariant to augmentations, while high-influence points suggest cases where augmentations produce unexpectedly large changes in representation space. This allows identification of examples that significantly shape the model's invariance properties. 

While our main definition uses cosine loss, we prove in Appendix~\ref{suppl:theory} that the Euclidean loss in Theorem~\ref{thm:inf-ssl} is proportional to cosine similarity for small perturbations, providing additional theoretical grounding for our choice of using cosine distance in our influence definition in Equation~\ref{eq:inf-ssl}. 

A natural question arising from our analysis is why we demonstrate the relationship between Influence-SSL and representation invariance only in the linear setting. While influence functions were originally developed with strong theoretical guarantees for simple models~\cite{hampel1974influence}, recent work has shown that in neural networks, they can deviate significantly from leave-one-out retraining effects~\cite{basu-inf} and may track fundamentally different quantities (referred to as the $\emph{Proximal Bregman Response Function}$)~\cite{bae2022if}. Our focus on the linear setting therefore serves two purposes. First, it provides a theoretically rigorous foundation where we can obtain closed-form solutions that precisely characterize what Influence-SSL measures. Second, following a rich tradition of analyzing learning dynamics in linear networks~\cite{misiakiewicz2023lectureslinearizedneuralnetworks}, this simplified setting offers geometric insights that help explain empirical observations in deep networks (see Section \ref{sec:expts}). An interesting direction for future work would be to theoretically characterize exactly what Influence-SSL approximates in non-linear architectures, similar to recent analyses for supervised influence functions~\cite{bae2022if}.
\begin{figure}
    \centering
    \includegraphics[width=\linewidth]{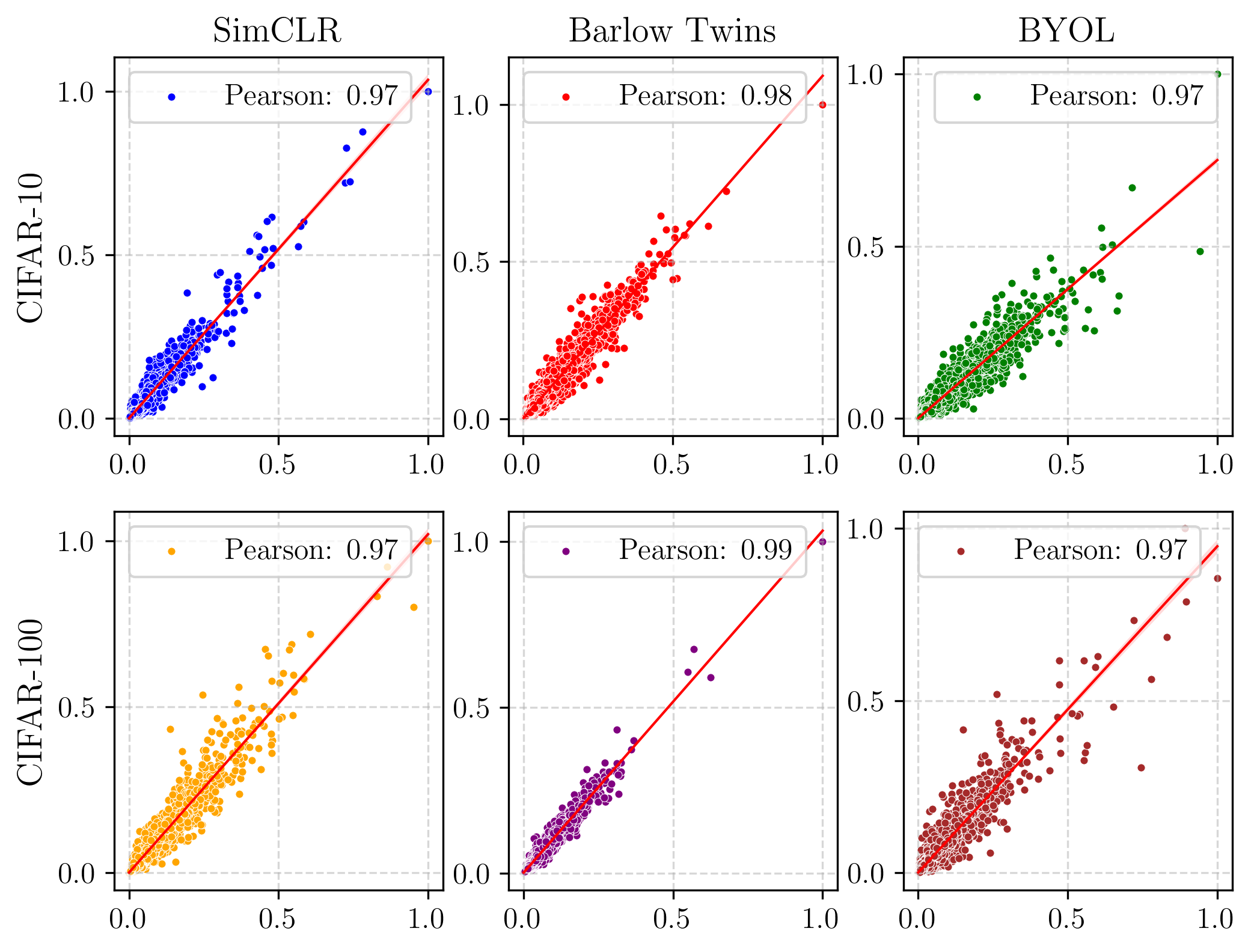}
    \vspace{-0.7cm}
    \caption{Correlation of Influence-SSL scores for two independent runs on SimCLR, Barlow Twins, and BYOL. Consistently high correlations are observed between influence scores computed using different initializations.}
    \label{fig:stable1}
\end{figure}

\begin{figure*}[t]
    \centering
    \includegraphics[width=\linewidth]{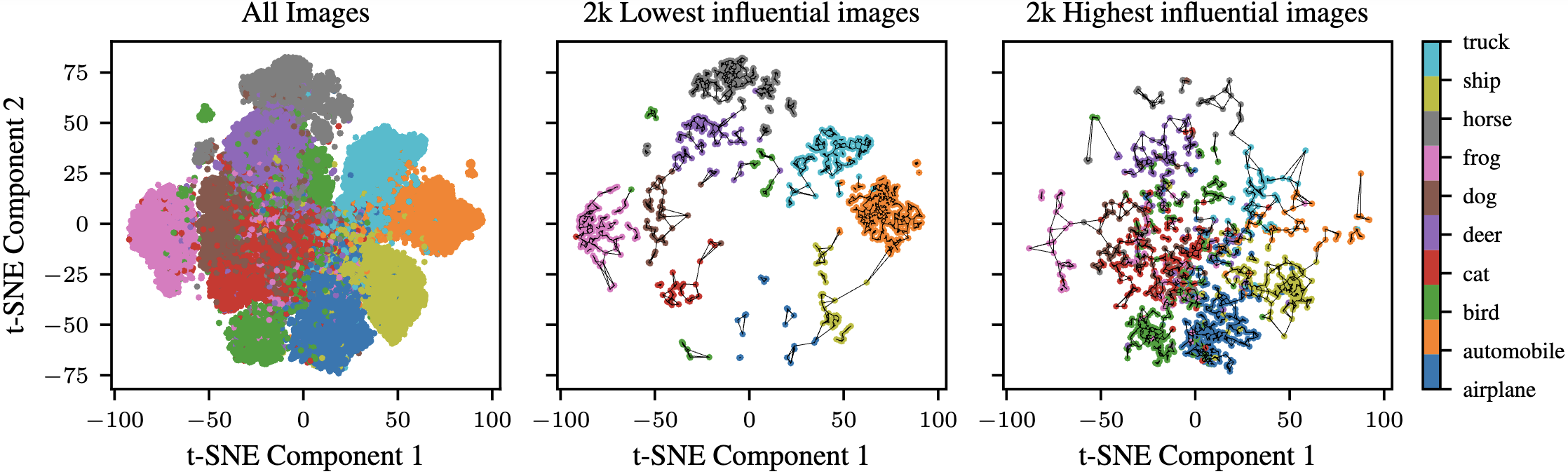}
    \caption{t-SNE projection of CIFAR-10 training images shows all examples (left), the 2,000 lowest (middle), and highest (right) influence scores. Low-influence images cluster tightly, while high-influence ones are dispersed.}
    \label{fig:tsne}
    \vspace{-0.5cm}
\end{figure*}

\subsection{Influence-SSL Properties in Linear Setting}

Building upon our previous analysis, we establish key properties of the Influence-SSL function that illuminate its geometric interpretation and behavior. We present intuitive explanations here, with formal theorem statements and proofs provided in Appendix~\ref{suppl:theory}.

\begin{proposition}[Structural Properties of Influence-SSL]
The Influence-SSL function exhibits the following properties:
\begin{enumerate}
\item \textbf{Geometric Decomposition}: $I_{ssl}(x) = -2\varepsilon^2\Tr(W\delta(x)\delta(x)^TW^T)$, separating into a perturbation scale factor and a representation sensitivity term
\item \textbf{Representation Invariance}: For any orthogonal matrix $Q$, $I_{ssl}(x;W) = I_{ssl}(x;QW)$
\item \textbf{Scaling Behavior}: $I_{ssl}(x;\alpha W) = \alpha^2 I_{ssl}(x;W)$ and $I_{ssl}(x;\varepsilon) = \varepsilon^2 I_{ssl}(x;1)$
\item \textbf{Stability}: $|I_{ssl}(x;W + E) - I_{ssl}(x;W)| \leq 4\varepsilon^2|\delta(x)|^2|W|_F|E|_F$ for perturbation $E$
\end{enumerate}
\end{proposition}

The geometric decomposition factors Influence-SSL into two components: perturbation magnitude ($\varepsilon^2$) and model sensitivity ($\Tr(W\delta(x)\delta(x)^TW^T)$), enabling separate analysis of augmentation strength and representation learning effects. The representation invariance property shows that influence measures are preserved under rotational transformations, capturing intrinsic representational properties independent of parameterization. This invariance, combined with scaling properties and stability bounds, demonstrates alignment with SSL's core objective of learning robust representations while ensuring predictable behavior under parameter perturbations. We also show compositional properties of Influence-SSL in Appendix~\ref{suppl:theory}. 

\section{Experiments}

\label{sec:expts}
In this section, we present experiments empirically validating some properties that were discussed in Section~\ref{sec:inf} including the consistency of Influence-SSL across methods and independent runs, and its ability to detect atypical data points. We also provide qualitative insights into the learned representations and their relationship with pretraining data.
\vspace{-7mm}

\paragraph{Experimental Setup.} To ensure comprehensive evaluation across different self-supervised learning paradigms, we experiment with three representative approaches: SimCLR (contrastive)~\cite{Chen2020ASF}, BYOL (distillation-based)~\cite{Grill2020BootstrapYO}, and Barlow Twins (invariance-based)~\cite{Zbontar2021BarlowTS}. Due to the computational complexity of inverse Hessian-vector product (IHVP) calculations, we employ ResNet18 as our backbone architecture, conducting experiments on CIFAR10 and CIFAR100 datasets~\cite{cifar}. Our implementation of Influence-SSL utilizes an efficient low-rank approximation for inverse Hessian estimation. For the perturbation scheme described in Section 3, we employ Gaussian noise with $\mu = 0.05$ and $\sigma = 0.2$, which we found to consistently identify influential examples across all SSL methods (additional perturbation experiments are detailed in Appendix~\ref{suppl:exp}). Unless explicitly stated otherwise, all experiments follow this standard configuration. To calculate influence scores, we adapt the LogIX package~\cite{Choe2024WhatIY} for SSL.

\subsection{Stability of Influence-SSL}

A critical consideration in influence estimation is the reproducibility and stability of the computed scores across different training runs. While traditional supervised learning influence methods have shown varying degrees of stability, the complex nature of contrastive and non-contrastive SSL frameworks raises additional concerns about the reliability of influence measurements. 

Building upon the theoretical guarantees established in Proposition 1, which demonstrates the stability of our influence estimation under certain conditions, our empirical analysis confirms these theoretical insights by showing notable stability in influence scores across independent runs. As shown in Figure~\ref{fig:stable1}, we observe consistently high \textit{Pearson rank} correlation ($\rho >0.96$) between influence scores computed from different initializations for all three frameworks: SimCLR, BYOL, and Barlow Twins. This strong correlation holds across both CIFAR-10 and CIFAR-100 datasets, suggesting that the influence estimation procedure captures reproducible patterns of sample importance rather than artifacts of specific training trajectories.


\subsection{Characteristics of Influential Examples}
\begin{figure}[t]
    \centering
    \includegraphics[width=\linewidth]{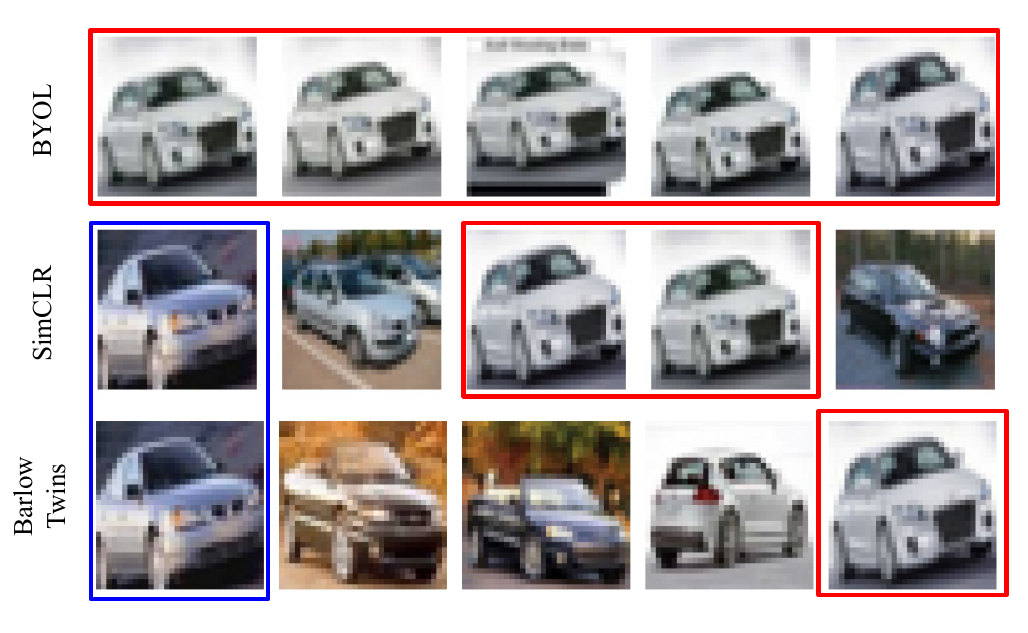}
    \vspace{-0.7cm}
    \caption{Five lowest influential images of the CIFAR-10 `automobile' class identified using BYOL, SimCLR, and Barlow Twins. Duplicate images are highlighted in red color.}
    \label{fig:duplicate}
    \vspace{-0.5cm}
\end{figure}

Our investigation into the characteristics of influential examples begins with a representation-level analysis. We visualize the learned representations through t-SNE as shown in Figure~\ref{fig:tsne}. We observe that the top-1000 influential examples form tightly intermingled clusters, while the lowest-1000 influential examples exhibit clear, well-separated clusters. This clustering behavior provides an initial insight into how Influence-SSL identifies examples based on their semantic relationships. The low-influential examples appear to possess clear, easily distinguishable semantic characteristics, making their class-level features more discriminative. In contrast, high-influential examples demonstrate semantic ambiguity, suggesting they play a more complex role in the self-supervised learning process.
\vspace{-4mm}
\paragraph{Can Influence Scores Identify Semantic Duplicates?}
Building upon representation-level insights from Figure~\ref{fig:tsne}, we discover an intriguing capability of our influence estimation method. Across different self-supervised learning frameworks, the influence scores consistently identify duplicate or near-duplicate images within classes, particularly evident in the automobile category of CIFAR-10. When examining the 5 lowest influential images for the automobile class on CIFAR-10 as shown in Figure~\ref{fig:duplicate}, we observed notable differences between the methods themselves. For BYOL, all 5 of the lowest influential images were duplicates of the same car, indicating that BYOL was particularly effective at identifying these near-identical images within the class. In contrast, SimCLR had 2 duplicate images in the top 5, while Barlow Twins only had 1 duplicate. Although SimCLR and Barlow Twins were less prone to selecting exact duplicates as the least 5-influential examples, further analysis revealed that more duplicate images emerged when the examination was extended beyond the top five least influential instances. An expanded version of Figure~\ref{fig:duplicate} with additional duplicates is in Appendix~\ref{suppl:exp}. 

\begin{figure}[t]
    \centering
    \includegraphics[width=\linewidth]{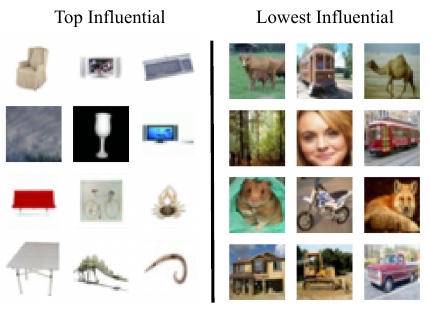}
    \vspace{-0.7cm}
    \caption{Visualization of the 12 highest and 12 lowest influential examples in CIFAR-100, showing that the highest influence images predominantly feature uniform backgrounds.}
    \label{fig:cifar100}
    \vspace{-0.6cm}
\end{figure}
\begin{figure*}[t]
    \centering
    \includegraphics[width=\linewidth]{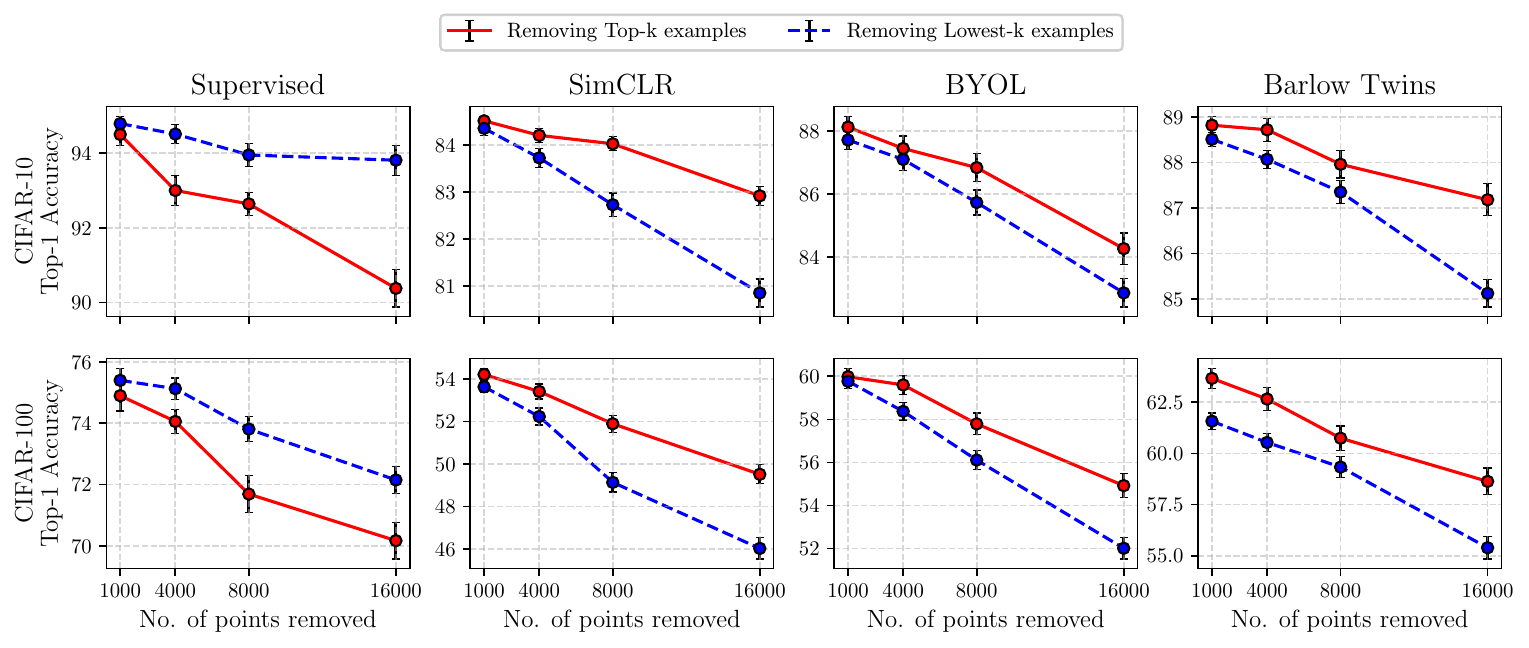}
    \vspace{-0.7cm}
    \caption{Accuracy (y-axis) vs. number of removed samples (x-axis) for CIFAR-10 (top row) and CIFAR-100 (bottom row). Samples were removed based on influence scores from pretrained SSL methods (SimCLR, BYOL, and Barlow Twins) and a supervised model. Unlike SSL, the supervised setting exhibits an opposite performance trend.}
    \label{fig:removal}
    \vspace{-0.5cm}
\end{figure*}

More importantly, these observations highlight an interesting and potentially useful property of Influence-SSL for identifying duplicates that was not explicitly modeled, but emerged naturally from the analysis. Understanding how different SSL methods differentially treat duplicate and outlier examples within classes could provide insights into their inner workings and lead to performance improvements or more robust model selection. Overall, this work demonstrates the value of influence estimation as a tool for probing the representations learned by self-supervised models.

\vspace{-3mm}
\paragraph{Visual patterns in influential examples.}
A particularly revealing pattern emerges when examining the visual characteristics of influential examples in CIFAR-100. Figure~\ref{fig:cifar100} shows that images with the highest influence scores predominantly feature uniform backgrounds (either white or black), deviating significantly from the natural image distribution. This finding provides critical insights into how self-supervised models learn and what they consider ``similar''. Unlike supervised learning, where class-level feedback guides the learning of discriminative features, SSL methods operate by bringing similar representations together based on different views of the same content. Ideally, these models should learn to focus on the semantic content (objects) while being invariant to background variations. However, our analysis reveals a different reality: the high concentration of uniform background images among influential examples, despite representing different object classes, suggests that the SSL model is inadvertently using background characteristics as a strong similarity signal.

This observation is particularly significant because it indicates a form of representational bias in SSL training. Images with similar backgrounds (especially uniform white or black) are being drawn together in the representation space, even when they contain semantically different objects. This background-driven clustering creates ambiguous decision boundaries in the representation space, as the model struggles to balance between background similarity and object-level semantic differences. Such a phenomenon wouldn't be easily detectable through conventional evaluation metrics, highlighting the value of influence analysis in understanding the biases learned during self-supervised training.

\subsection{Impact of Influential Examples on Model Performance: A Counterintuitive Discovery}

The conventional wisdom in machine learning suggests that removing highly influential training examples should lead to a significant degradation in model performance~\cite{feldmanzhang, feldman2020does, hammoudeh23}. This intuition, well-established in supervised learning literature, is based on the premise that influential examples play a crucial role in defining decision boundaries. However, our analysis in the self-supervised learning context reveals a striking and counterintuitive phenomenon.

As shown in Figure~\ref{fig:removal}, when we progressively remove the most influential examples from the training set, SSL methods (SimCLR, BYOL, and Barlow Twins) exhibit improved downstream performance, in direct contrast to the supervised learning baseline which shows expected degradation. This surprising result can be understood through the lens of our previous findings about the nature of high-influence examples in SSL. We posit that these examples, predominantly featuring uniform backgrounds, act as ``semantic bridges'' in the representation space – artificially connecting instances from different classes based on their background similarities rather than their semantic content.

The removal of these high-influence examples effectively eliminates these potentially misleading bridges, allowing the SSL models to focus more on genuine semantic relationships. This ``purification'' of the training set leads to more discriminative representations, as the model is no longer compelled to reconcile the conflicting signals between background similarity and semantic difference. In contrast, removing low-influence examples, which typically have more natural backgrounds and clearer semantic content, results in the expected performance degradation across all methods.

This finding challenges our fundamental assumptions about influence in representation learning and suggests that traditional notions of example importance may need to be reconsidered in the context of self-supervised learning. While high-influence examples in supervised learning often contribute positively to model performance by helping define class boundaries, their role in SSL appears to be more complex and potentially detrimental when they introduce unintended biases through background characteristics.

\subsection{Theoretical Intuition: Characterizing High-Influence Points via Augmentation Effects}

To provide theoretical intuition for our empirical findings that removing high-influence points improves downstream performance, we appeal to our simplified linear setting described earlier. 

For ease of exposition, we make the stochasticity in our augmentation framework explicit where each input x is transformed as $x_{aug} = x + \varepsilon\delta(x,\xi)$, where $\delta(x,\xi)$ is an input-dependent perturbation of unit norm, $\xi$ is drawn from a distribution $P(\xi)$ capturing randomness in the augmentation process, and $\varepsilon \ll 1$ controls the perturbation magnitude. Our goal is to understand when and why certain training points exhibit high influence under this framework.

We begin by examining the expected influence across all inputs:

\begin{definition}[Expected SSL Influence]
For a model parameterized by matrix $W$, the expected influence under augmentations is:
\begin{align*}
    I_{expected}(W) &= -2\varepsilon^2\mathbb{E}_{x \sim P(x), \xi \sim P(\xi)}[|W\delta(x,\xi)|^2] \\
                    &=  -2\varepsilon^2\Tr(W^TW\Sigma)
\end{align*}
where $\Sigma = \mathbb{E}_{x,\xi}[\delta(x,\xi)\delta(x,\xi)^T]$ captures the second moment of perturbations across all inputs.
\end{definition}

This expected influence serves as a reference point - points with higher magnitude influence than this expectation are those where augmentations have an unusually strong effect. The following proposition characterizes this relationship:

\begin{proposition}[Influence Deviation Characterization]
For a training point $x$ with influence $I_{ssl}(x) = -2\varepsilon^2|W\delta(x,\xi)|^2$, the deviation from expected influence is:
   \begin{align*}
       I_{ssl}(x) - &\mathbb{E}_{\xi \sim P(\xi)}[I_{ssl}(x)] \\
       &= -2\varepsilon^2\Tr(W^TW(\delta(x,\xi)\delta(x,\xi)^T - \Sigma_x)) 
   \end{align*}
where $\Sigma_x = \mathbb{E}_{\xi}[\delta(x,\xi)\delta(x,\xi)^T]$ represents the expected augmentation behavior for input $x$.
\end{proposition}

This characterization shows exactly when a training point has unusually high influence. The term $\delta(x,\xi)\delta(x,\xi)^T - \Sigma_x$ measures how much a specific augmentation of point $x$ differs from its typical augmentation behavior ($\Sigma_x$). When this difference interacts strongly with the model's transformation ($W^TW$), we get a point whose influence deviates significantly from expectation. In other words, high-influence points are those where augmentations produce unexpectedly large changes in the representation space. For this characterization to meaningfully identify atypical points, we require some regularity in how augmentations behave across the dataset (see technical conditions in Appendix~\ref{suppl:theory}.


Under these conditions, high-influence points correspond to examples where augmentations produce unexpectedly large deviations from expected behavior while aligning with the model's transformation. While our linear analysis cannot fully explain the downstream benefits of removing such points, it suggests these examples may interfere with learning consistent features due to their atypical augmentation behavior. Our empirical observations of improved task performance after removing high-influence points support this interpretation, though establishing rigorous theoretical connections between augmentation consistency and downstream performance remains an open challenge.

\subsection{Model Fairness}

\begin{figure}[t]
    \centering
    \includegraphics[width=\linewidth]{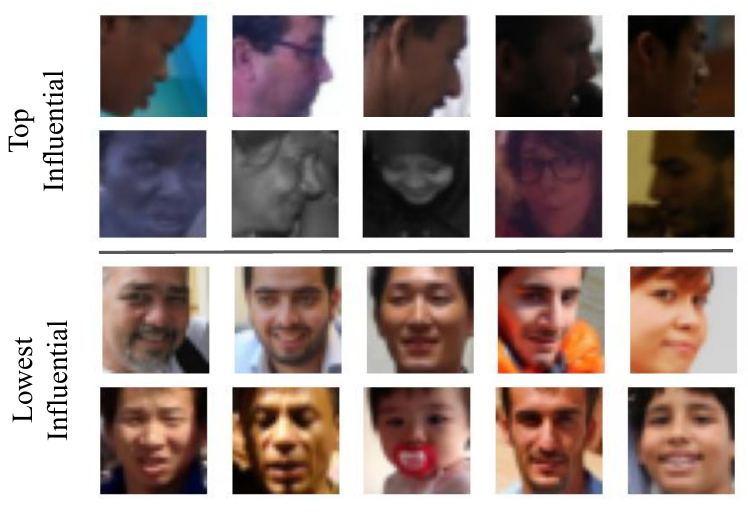}
    \vspace{-0.7cm}
    \caption{The top 10 and lowest 10 influential examples in the FairFace dataset were identified using BYOL, showing that the high influential examples are mostly challenging cases.}
    \label{fig:faces}
\end{figure}

\begin{figure}[t]
    \centering
    \includegraphics[width=\linewidth]{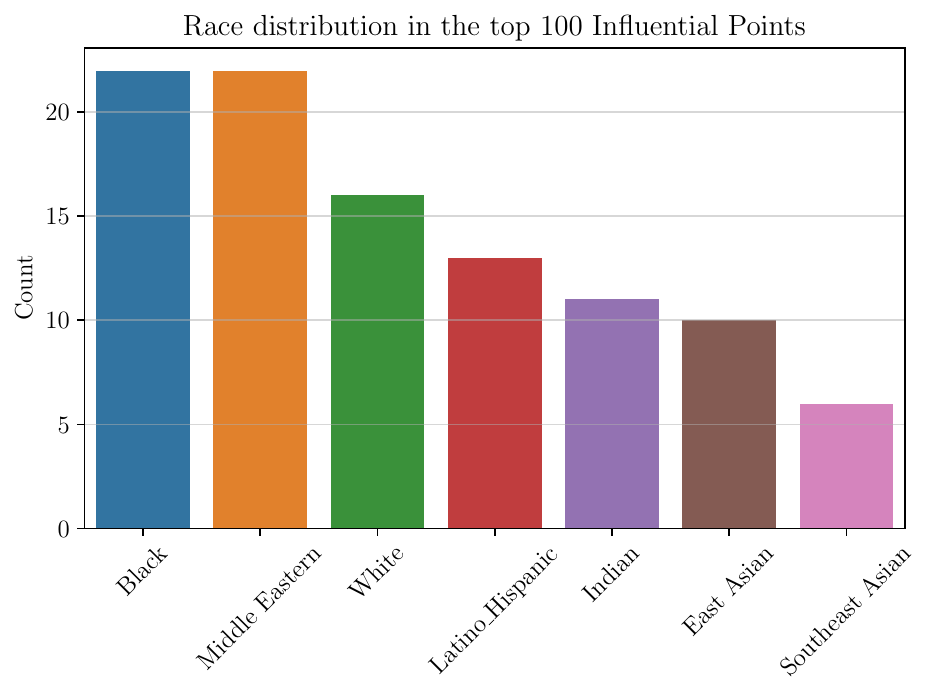}
    \vspace{-0.7cm}
    \caption{The histogram of the racial distribution among the top 100 influential examples in the FairFace dataset, revealing uneven representation.}
    \label{fig:face-bar}
\end{figure}

To analyze potential biases in self-supervised learning, we conducted experiments using the FairFace dataset~\cite{karkkainenfairface}, which was specifically designed to provide balanced racial representation across different demographic groups. We chose FairFace for its careful curation of facial images with balanced distributions across race, gender, and age groups, making it an ideal testbed for investigating potential biases in representation learning. We trained BYOL~\cite{Grill2020BootstrapYO} on this dataset for 200 epochs, using the 0.25-padding version of the dataset. Leveraging BYOL~\cite{Grill2020BootstrapYO} as our self-supervised learning framework, we examined the distribution of influential examples identified by our proposed Influence-SSL method.

Figure~\ref{fig:faces} shows the top and lowest 10 influential examples. We observed that highly influential examples predominantly consisted of challenging cases - images with significant pose variations (side-turned faces) and suboptimal lighting conditions. In contrast, low-influence examples consistently featured well-lit, front-facing portraits. More importantly, when examining the racial distribution among the top 100 influential examples, we found an uneven representation: Black and Middle Eastern faces were disproportionately represented among high-influence points, followed by White faces, while East and Southeast Asian faces were notably underrepresented (See Figure~\ref{fig:face-bar}). While our work does not directly address bias mitigation, we would like to show that Influence-SSL can be valuable for detecting and analyzing biases that may emerge during self-supervised training - biases that might remain hidden using conventional evaluation methods. This capability is particularly crucial as self-supervised learning continues to be widely adopted in various computer vision applications where fairness considerations are important.

\section{Discussion and Conclusion}
This work presents Influence-SSL, the first systematic approach to defining and measuring influence in self-supervised learning. Our method not only extends the theoretical foundations of influence estimation to the SSL paradigm but also reveals several surprising insights about how these models learn from training data. Our findings challenge a core established assumption about influential examples in machine learning. We demonstrate that high-influence examples in SSL often exhibit characteristics that can potentially impede optimal representation learning – a sharp contrast to supervised learning where influential examples typically contribute positively to model performance. 

We show that Influence-SSL can be effectively used for identifying semantic duplicates within datasets, reveals potential biases in learned representations, and provides insights into fairness considerations, as demonstrated in our FairFace experiments. These capabilities make it a valuable tool for dataset curation, model debugging, and ensuring fairness in self-supervised learning systems. While our current study focuses on CIFAR-10 and CIFAR-100 datasets, the principles and methodologies we've established can be extended to larger-scale datasets and diverse domains.  Additionally, our findings about the role of background features in SSL suggest potential directions for developing more robust self-supervised learning algorithms that better capture semantic relationships while being less susceptible to spurious correlations.

{
    \small
    \bibliographystyle{ieeenat_fullname}
    \bibliography{main}
}

\clearpage
\setcounter{page}{1}
\maketitlesupplementary
\appendix


The appendix provides supplementary material to enhance the reproducibility, depth, and theoretical understanding of our study. Readers can navigate these sections as described below, to explore specific topics in greater detail:
\begin{itemize}
    \item \textbf{Section~\ref{suppl:rep}} details the steps and resources to ensure the reproducibility of our results.
    \item  \textbf{Section~\ref{suppl:exp}} includes additional experiments with deeper insights of Influence-SSL across various settings: comparisons of distributions (B.1), ablation studies on perturbation choice (B.2) and strength (B.3), an extended analysis of duplicate instances in CIFAR-10 (B.4), more visualizations of influential examples across CIFAR-100 (B.5). 
    \item \textbf{Section~\ref{suppl:theory}} provides proofs and additional theoretical discussions, including the proof of Theorem 1 (C.1), an exploration of the properties of Influence-SSL (C.2), and additional details from Section 4.4 of the main paper (C.3). 
\end{itemize}

\section{Reproducibility}
\label{suppl:rep}
Throughout our experiments, we utilize 3 self-supervised models: SimCLR, BYOL and Barlow. We obtain pre-trained weights provided by the Solo-learn library\footnote{Solo-learn: \url{https://github.com/vturrisi/solo-learn}} for computing Influence-SSL scores. Specifically, all the pretrained models for CIFAR-10 and CIFAR-100 datasets are trained with configurations detailed in Table~\ref{tab:pretrain}. We have used a consistent set of augmentations and hyperparameters across all pre-training setups, as shown in Table \ref{tab:augmentations}. 

Since we are only interested in the influence of training data on the model behavior, we only consider training set from CIFAR-10, CIFAR-100 and FairFace datasets both for pre-training and computing Influence-SSL scores. Validation ses from CIFAR-10 and CIFAR-100 are only used for obtaining the best checkpoints for training the supervised models. The pre-training configuration for the high and low-influential removal experiment of Figure~\ref{fig:removal} is detailed in Table~\ref{tab:removal_config}. We also share the seeds used for computing the error bars of Figure~\ref{fig:removal} in Table~\ref{tab:removal_config} set using:   

\begin{tcolorbox}
\begin{verbatim}
import lightning.pytorch as pl

SD = 0 #Seed number
pl.seed_everything(
         SD,
         workers=True
)
\end{verbatim}
\end{tcolorbox}

\begin{table}[ht]
\centering
\caption{Data augmentation configuration for pre-training.}
\label{tab:augmentations}
\begin{tabular}{lc}
\toprule
\textbf{Augmentation}       & \textbf{Hyperparam}
             \\ \midrule
Brightness & 0.8\\
Contrast & 0.8 \\
Saturation & 0.8 \\
Hue & 0.2 \\
Color Jitter Probability    & 0.8 \\
Grayscale Probability       & 0.2 \\
Horizontal Flip Probability & 0.5 \\
Gaussian Blur Probability   & 0.2 \\
Solarization Probability    & 0.2 \\
Crop Size                   & 32 \\
Random Crop - Min Scale               & 0.08 \\
Random Crop - Max Scale               & 1.0 \\ \bottomrule
\end{tabular}
\end{table}

\begin{table*}[ht]
\centering
\caption{Pre-training configuration for models before computing influence scores.}
\label{tab:config_details}

\begin{tabular}{@{}lccc@{}}
\toprule
\textbf{Configuration}        & \textbf{SimCLR}                              & \textbf{BYOL}                              & \textbf{Barlow Twins}                      \\ \midrule
\textbf{Backbone}             & ResNet18                                     & ResNet18                                   & ResNet18                                   \\ 
\textbf{Dataset} & CIFAR-10, CIFAR-100 & CIFAR-10, CIFAR-100 & CIFAR-10, CIFAR-100 \\ 
\textbf{Projection Hidden}      &  2048&  4096 & 2048\\ 
\textbf{Projection Output}      &  256 & 256 & 2048         \\ 
\textbf{Max Epochs}           & 1000                                         & 1000                                       & 1000                                       \\ 
\textbf{Precision}            & Mixed-16                                     & Mixed-16                                         & Mixed-16                                         \\ 

\textbf{Batch Size}           & 256                                          & 256                                        & 256                                        \\ 
\textbf{Learning Rate (LR)}   & 0.4                                          & 1.0                                        & 0.3                                        \\ 
\textbf{Classifier LR}        & 0.1                                          & 0.1                                        & 0.3                                        \\ 
\textbf{Weight Decay}         & $1 \times 10^{-5}$                           & $1 \times 10^{-5}$                         & $1 \times 10^{-4}$                         \\ 
\textbf{Optimizer}            & SGD (LARS Enabled)                           & SGD (LARS Enabled)                         & SGD (LARS Enabled)                         \\ 
\textbf{LARS Eta}             & 0.02                                         & 0.02                                       & 0.02                                       \\ 
\textbf{Scheduler}            & Warmup Cosine                                & Warmup Cosine                              & Warmup Cosine                              \\ 
\textbf{Warmup Start LR}      & 0.003                                        & 0.003                                      & 0.003                                      \\ 
\textbf{Warmup Epochs}        & 10                                           & 10                                         & 10                                         \\ 
\textbf{Temperature}          & 0.2                                          & N/A                                        & N/A                                        \\ 
\textbf{Momentum} & 0.9 & 0.9 & 0.9 \\ 
\textbf{Pretrained Checkpoints} & \hyperlink{https://drive.google.com/drive/folders/1mcvWr8P2WNJZ7TVpdLHA_Q91q4VK3y8O?usp=sharing}{cifar10}, \hyperlink{https://drive.google.com/drive/folders/13pGPcOO9Y3rBoeRVWARgbMFEp8OXxZa0}{cifar100} & \hyperlink{https://drive.google.com/drive/folders/1KxeYAEE7Ev9kdFFhXWkPZhG-ya3_UwGP}{cifar10}, \hyperlink{https://drive.google.com/drive/folders/1hwsEdsfsUulD2tAwa4epKK9pkSuvFv6m}{cifar100} & \hyperlink{https://drive.google.com/drive/folders/1L5RAM3lCSViD2zEqLtC-GQKVw6mxtxJ_}{cifar10}, \hyperlink{https://drive.google.com/drive/folders/1hDLSApF3zSMAKco1Ck4DMjyNxhsIR2yq}{cifar100} \\ \bottomrule

\end{tabular}\label{tab:pretrain}
\end{table*}

\paragraph{Note on FairFace experiment:} Several commercial computer vision systems (Microsoft, IBM, Face++) have been criticized due to their asymmetric accuracy across sub-demographics in recent studies~\cite{buolamwini2018gender, raji2019actionable}. These studies have shown that the commercial face processing systems all perform better on some races and light faces. This can be caused by the biases
in their training data and biases picked up during learning. FairFace dataset~\cite{karkkainenfairface}  has been created towards tackling algorithmic fairness, and in our work, we aim to identifying data with potential for creating biases in self-supervised learning methods.

\begin{table*}[ht]
\centering
\caption{Pre-training configuration of models for the removal experiment in Figure 6.}
\label{tab:removal_config}
\begin{tabular}{lccc}
\toprule
\textbf{Configuration} & \textbf{SimCLR} & \textbf{BYOL} & \textbf{Barlow Twins} \\ \midrule
\textbf{Backbone}       & ResNet18        & ResNet18      & ResNet18             \\ 
\textbf{Dataset}        & CIFAR-10, CIFAR100        & CIFAR-10, CIFAR100      & CIFAR-10, CIFAR100             \\ 
\textbf{Projection Hidden} & 2048 & 2048 & 2048  \\ 
\textbf{Projection Output} & 256 &
 256 & 512 \\ 
\textbf{Loss Temperature}    & 0.2   &  0.2  & 0.1   \\ 
\textbf{Optimizer}      & LARS  & LARS  & LARS  \\ 
\textbf{LARS Eta}             & 0.02                                         & 0.02                                       & 0.02                                       \\ 
\textbf{Initial Learning Rate}      &  0.4 & 0.35 & 0.35  \\ 
\textbf{Weight Decay}      & $1\times 10^{-4}$ & $3\times10^{-5}$ & $3\times10^{-5}$  \\ 
\textbf{Batch Size}     & 512    & 512 & 512\\ 
\textbf{Scheduler}   & Warmup Cosine  & Warmup Cosine& Warmup Cosine  \\ 
\textbf{Warmup Epochs}        & 10                                           & 10                                         & 10                                         \\ 
\textbf{Max Epochs}         & 250 & 250 & 250 \\ 
\textbf{Precision}      & Mixed-16        & Mixed-16       & Mixed-16            \\ 
\textbf{Momentum} & 0.9 & 0.9 & 0.9 \\ 
\textbf{Seeds}      & 000, 042, 123        & 000, 042, 123       & 000, 042, 123           \\ 
\bottomrule
\end{tabular}
\end{table*}

\section{Additional Experiments}
\label{suppl:exp}
\vspace{0 mm}
\subsection{Influence-SSL: Distributional Comparison}
\begin{figure}[t]
    \centering
    \includegraphics[width=\linewidth]{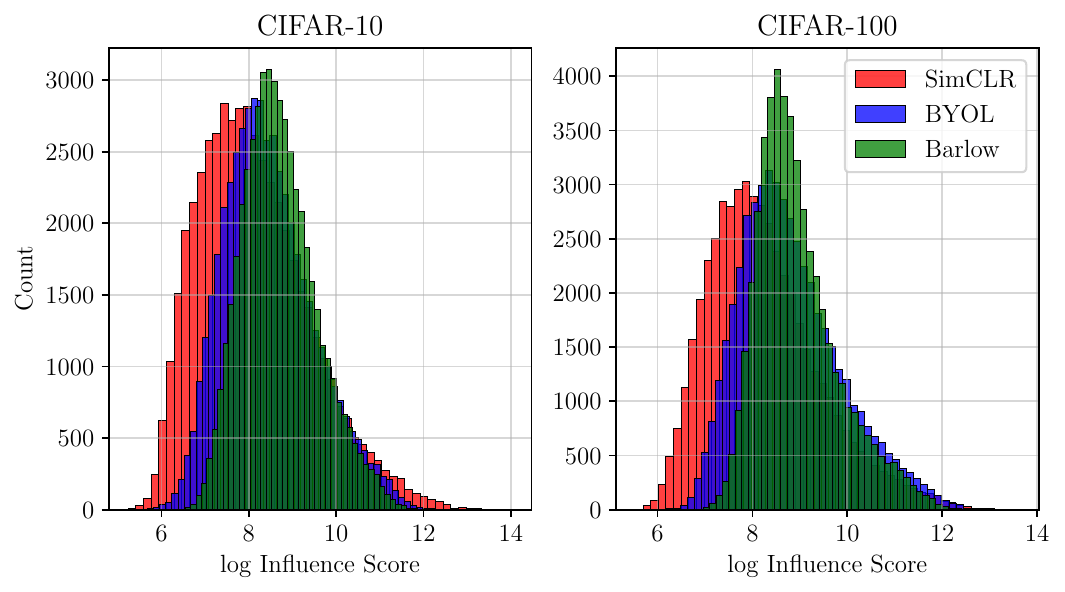}
    \caption{Distribution comparison of $\log$ Inluence-SSL scores between SSL methods like SimCLR, BYOL and Barlow Twins.}
    \label{fig:supl-hist}
\end{figure}
When comparing the distributions of influence scores across different self-supervised learning frameworks such as SimCLR, BYOL, and Barlow Twins on CIFAR-10 and CIFAR-100, we observe distinct patterns in their $\mu$ and $\sigma$ values. For CIFAR-10, Barlow Twins achieves the highest $\mu$ influence score (8.71) with the lowest $\sigma$ (0.85), followed by BYOL with a $\mu$ of 8.47 and a $\sigma$ of 1.03. SimCLR exhibits the lowest $\mu$ (8.13) and the highest variability (1.27). A similar trend is observed for CIFAR-100 but more apparent, where Barlow Twins again leads with the highest $\mu$ score (8.96) and lowest variability (0.86), followed by BYOL ($\mu$: 8.81, $\sigma$: 1.08) and SimCLR ($\mu$: 8.24, $\sigma$: 1.18).

Interestingly, these results align with the top-1 accuracy rankings: Barlow Twins > BYOL > SimCLR, suggesting a correlation between higher influence scores and better downstream performance. The lower standard deviation in Barlow Twins indicates more stable and consistent influence scores across samples, which may contribute to its superior performance. BYOL, with slightly higher variability, balances strong influence scores with moderate stability. In contrast, SimCLR shows both the lowest mean and highest variability, potentially reflecting its reliance on simpler contrastive objectives that may lead to less consistent influence across representations.

\subsection{Ablation: Influence-SSL Perturbation Choice}
\begin{table}
\centering
\caption{Pearson Ranked Correlation for Different Augmentations.}

\begin{tabular}{l c}
\toprule
\textbf{Perturbation ($\epsilon$)} & \textbf{Pearsonr} \\ 
\midrule 
Random Flip + Grayscale + Gauss Blur   & 0.7197 \\ 
Random Flip + Color Jitter + Grayscale & 0.7234 \\ 
Random Flip + Color Jitter + Gauss Blur & 0.7422 \\ 
Random Flip + Gaussian Blur               & 0.7447 \\ 

Random Cropping                               & 0.9509  \\
Gaussian Noise (original Influence-SSL) & \textbf{0.9745} \\
\bottomrule
\end{tabular}
\label{tab:abl-aug}
\end{table}

In the experiments presented in the main paper, we modeled the perturbation ($\epsilon$) in $\hat{x} = x + \epsilon$ as Gaussian noise. This choice was motivated by its simplicity for theoretical argument and the consistently high \textit{Pearson} rank correlation observed across a range of hyperparameters, datasets, and SSL frameworks. Gaussian noise serves as an intuitive and effective baseline, enabling us to explore the influence of perturbations on learned representations without introducing unnecessary complexity. However, the choice of perturbation type may significantly impact the influence scores of the self-supervised learning frameworks. In this section, we extend our analysis by exploring alternative augmentation strategies for generating perturbations. We aim to evaluate the stability of influence scores on different types of perturbations.

The results from evaluating different augmentation strategies as reported in Table~\ref{tab:abl-aug} reveal interesting insights into their effects on the Pearson rank correlation of influence scores. Among the augmentations tested, Gaussian Noise (GN) with a standard deviation of 0.2 achieves the highest correlation (0.974), confirming its stability and alignment with the baseline experiments in the main paper. This is closely followed by Random Cropping (RP) with a correlation of 0.951, suggesting that spatial perturbations can also preserve the relative order of influence scores effectively. In contrast, augmentations involving combinations of visual transformations such as Horizontal Flip (HF), Color Jittering (CJ), Gaussian Blur (GB), and Grayscale (GS) result in noticeably lower correlations, ranging from 0.719 to 0.744. These findings highlight the robustness of Gaussian Noise as a perturbation strategy, as well as the potential for spatial augmentations like random cropping to serve as viable alternatives. However, augmentations that heavily alter visual features (e.g., color and texture) tend to reduce correlation, likely due to their stronger impact on the input space.

\subsection{Ablation: Varying Influence-SSL Perturbation Strength}
Since Influence-SSL has the following form:
$$
\mathcal{I}(f, i) = -\nabla \mathcal{L}(f(x_i), f(\hat{x}_i))^\top H_\theta^{-1} \nabla \mathcal{L}(f(x_i), f(\hat{x}_i))
$$

where $\hat{x}_i = x_i + \mathcal{N}(\mu, \sigma)$. We want to discuss the effect of $\mu$ and $\sigma$ which serves as the hyper-parameters for the \textit{Gaussian} perturbation.
\begin{figure}
    \centering
    \includegraphics[width=\linewidth]{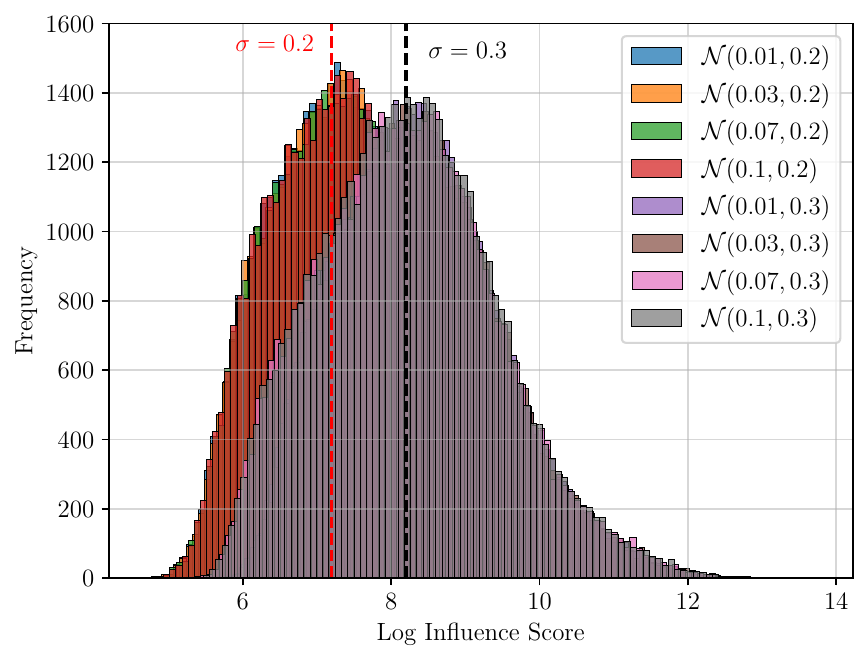}
    \caption{Distribution of influence scores with varying perturbation levels ($\mu$ and $\sigma$) of the \textit{Gaussian} noise.}
    \label{fig:abl-loss}
\end{figure}
\begin{figure}
    \centering
    \includegraphics[width=\linewidth]{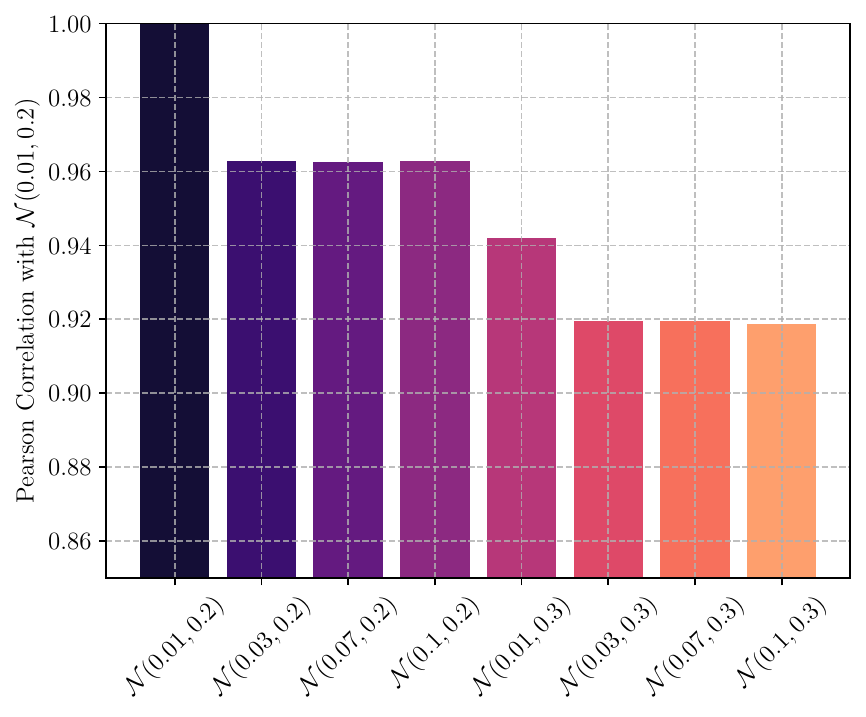}
    \caption{Pearson ranked correlation between influence scores from $\mathcal{N}(0.01, 0.2)$  with other perturbation levels.}
    \label{fig:abl-corr}
\end{figure}

We observe that when comparing small standard deviation values (0.2 and 0.3), the distributions of influence scores exhibit significant overlap, as illustrated in Figure~\ref{fig:abl-loss}. Specifically, when moving from a standard deviation of 0.2 to 0.3, the distribution shifts slightly to the right. This suggests that perturbations with larger deviations generally result in higher average influence scores. However, this shift does not necessarily lead to a substantial change in the relative ranking of the scores within the distribution.

To further analyze this, Figure~\ref{fig:abl-corr} presents a comparison of the \textit{Pearson} rank correlation between these scores and the baseline scores computed with a standard deviation of $\mathcal{N}(0.01, 0.2)$. Despite the observed shift in distributions, the scores remain highly correlated, with all correlations exceeding 0.9. Interestingly, higher perturbation deviations tend to slightly reduce the overall correlation, indicating a gradual impact on the rank ordering as the deviation increases. This demonstrates that while increased deviation affects the absolute scores, it does not disrupt the underlying relative order to a significant extent.

\subsection{Duplicate Instances in CIFAR-10}
\begin{figure}
    \centering
    \includegraphics[width=\linewidth]{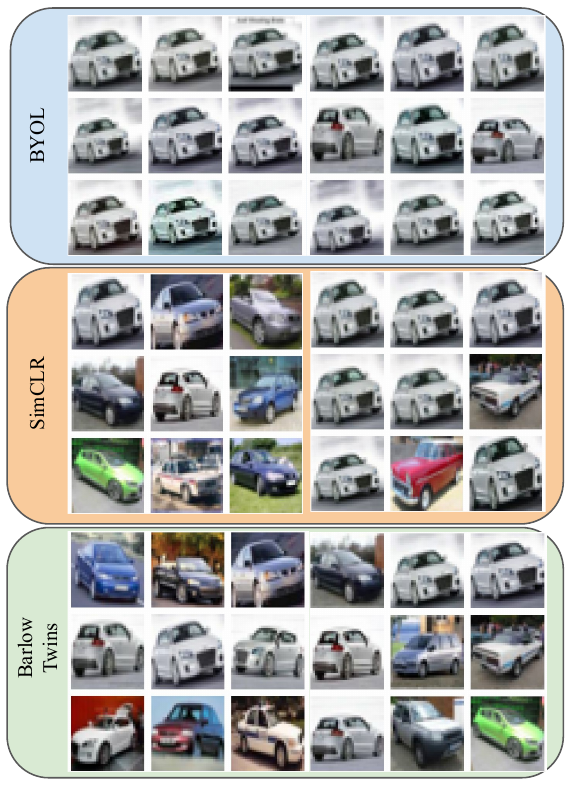}
    \caption{More examples of lowest-18 influential images in `automobile' class of CIFAR-10.}
    \label{fig:abl-dup}
\end{figure}
In our extended analysis of duplicates, we plot the 18 least influential images from the `automobile' class of CIFAR-10 to compare BYOL, SimCLR, and Barlow Twins. As we discussed in section~4.2, for BYOL, all the selected images are duplicates, representing just two distinct cars repeated multiple times. SimCLR exhibits slightly fewer duplicates than BYOL but still shows a high prevalence of repeated instances, while Barlow Twins has the least number of duplicates among the three methods. Despite these variations, most of the images across all methods are duplicates, suggesting that these frameworks assign low influence to repetitive instances, possibly due to their limited contribution to learning diverse representations. The higher concentration of duplicates in BYOL may reflect its reliance on pairwise consistency objectives, which could make it more sensitive to redundancy in the data. In contrast, SimCLR and Barlow Twins, particularly the latter, appear to mitigate the impact of duplicates more effectively, aligning with their design differences in handling data augmentation and feature alignment. This observation highlights the varying susceptibility of SSL frameworks to redundancy in the dataset.

\subsection{Visualization of Influential Examples}
We further analyze the qualitative differences between the most and least influential images for SimCLR (Figures~\ref{fig:top400-simclr} and \ref{fig:low400-simclr}) and supervised ResNet-18 (Figures~\ref{fig:top400-sup} and \ref{fig:low400-sup}) on the CIFAR-100 dataset by visualizing the top 400 and lowest 400 influential images for each method. In the case of SimCLR, we observe a distinct pattern among the top 400 influential images, many of which feature either predominantly white or black backgrounds. This suggests a deviation from the natural distribution of CIFAR-100, potentially indicating that SimCLR assigns high influence to outliers or instances that do not conform to the dataset’s typical visual characteristics. Conversely, the least influential images for SimCLR are notably more colorful and appear to represent easier examples, possibly due to their closer alignment with the underlying distribution and straightforward features.

For supervised ResNet-18, however, no such distinct trend is observed in the influential images. Since supervised learning involves labels, the definition of influence is tied to label-driven objectives, and highly influential images are often atypical from a label noise perspective rather than due to visual characteristics. This suggests that supervised models focus on resolving label inconsistencies or learning hard examples, whereas SimCLR prioritizes instances that deviate from the dataset’s image distribution.
\begin{figure*}[p]
    \centering
    \includegraphics[width=\linewidth]{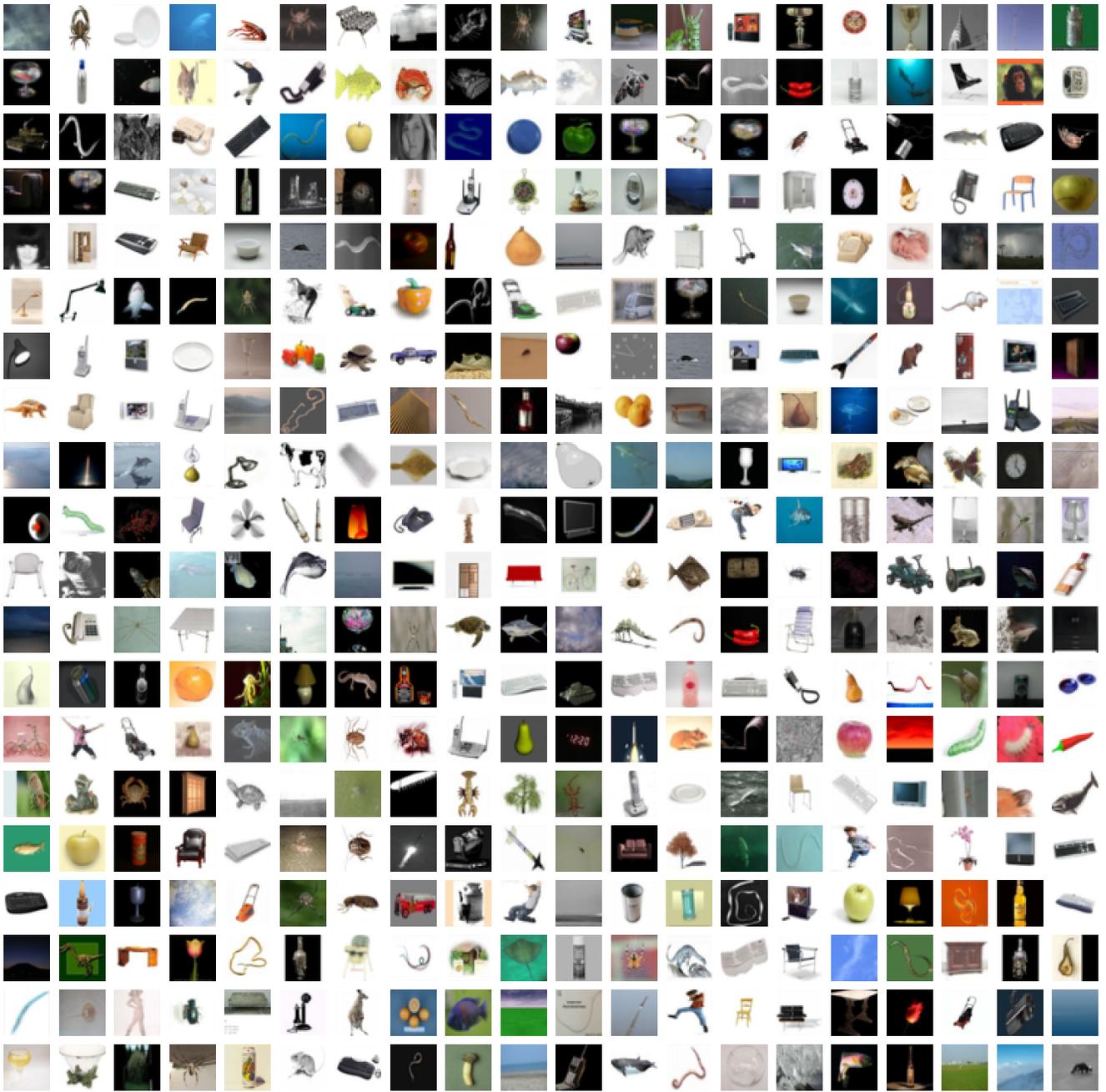}
    \caption{Top 400 influential images for SimCLR on CIFAR100.}
    \label{fig:top400-simclr}
\end{figure*}

\begin{figure*}[p]
    \centering
    \includegraphics[width=\linewidth]{sections/Figures/top-False.pdf}
    \caption{Lowest 400 influential images for SimCLR on CIFAR100.}
    \label{fig:low400-simclr}
\end{figure*}

\begin{figure*}[p]
    \centering
    \includegraphics[width=\linewidth]{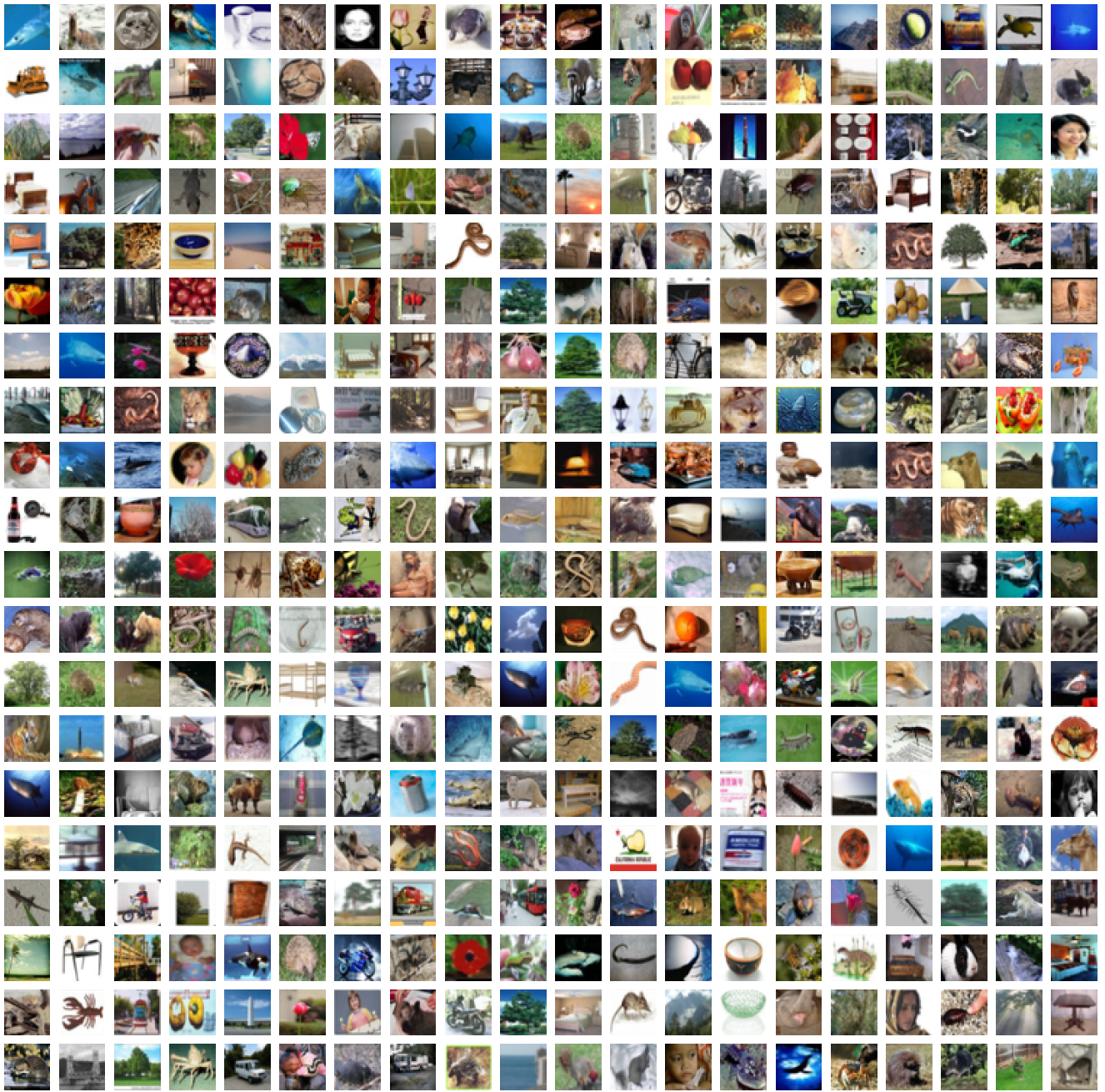}
    \caption{Top 400 influential images for supervised ResNet-18 on CIFAR100.}
    \label{fig:top400-sup}
\end{figure*}

\begin{figure*}[p]
    \centering
    \includegraphics[width=\linewidth]{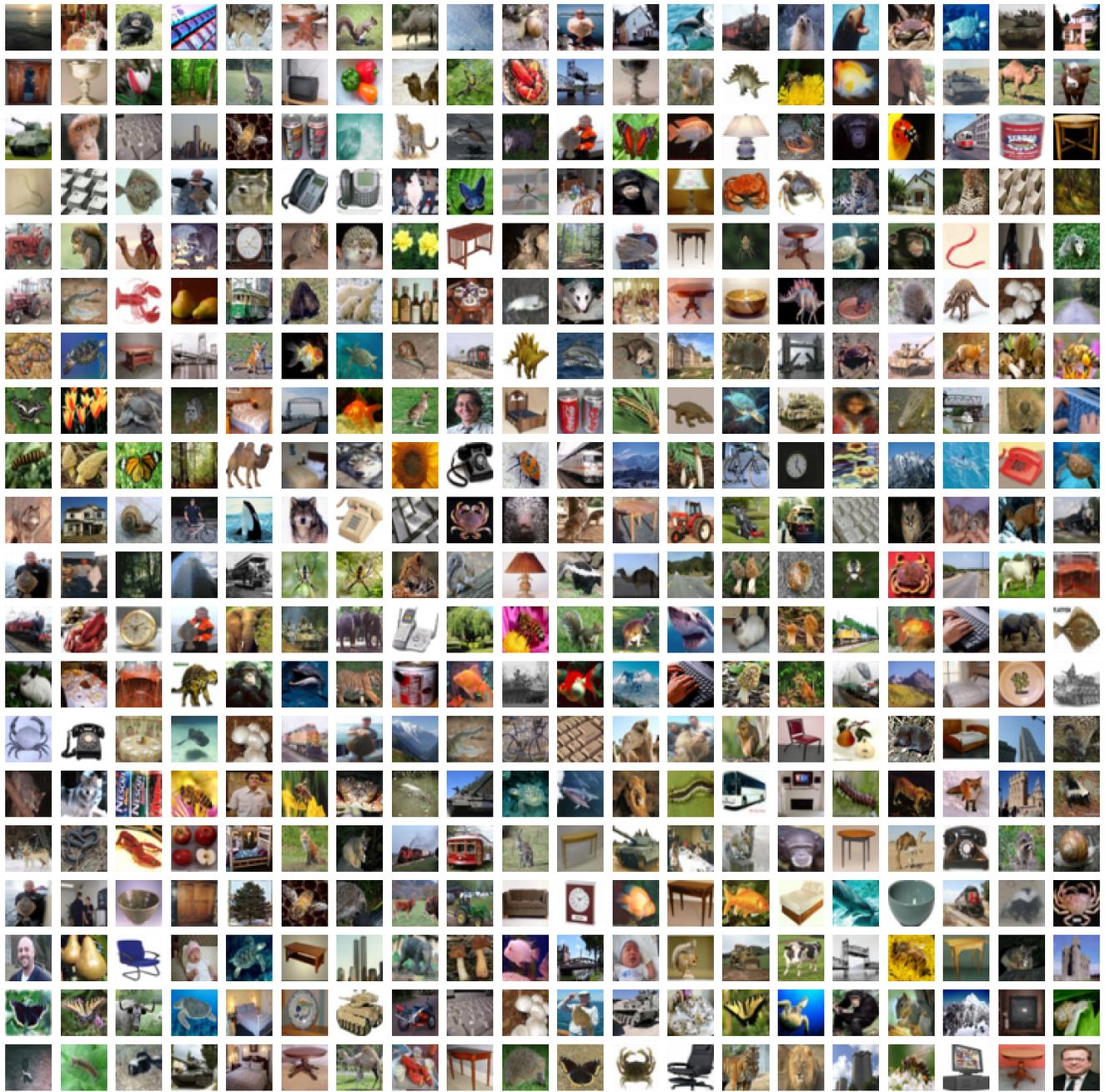}
    \caption{Lowest 400 influential images for supervised ResNet-18 on CIFAR100.}
    \label{fig:low400-sup}
\end{figure*}

\newpage\clearpage

\section{Theoretical Analysis} \label{suppl:theory}
In this section, we provide proofs and additional details of our theoretical results.

\subsection{Proof of Theorem 1}

We break this proof into several steps: first, we choose our setup to be simple enough such that we can get a closed-form solution for the influence function while still being insightful enough to give us some information of the underlying dynamics. Then, we compute the gradients and hessians relevant to our problem and explicitly compute the influence function according to our proposed definition.

In Proposition~\ref{prop:cosine}, we show a connection between the choice of Euclidean distance chosen in our theoretical setup in Theorem 1 and the choice of using cosine distance in our definition of influence in our proposed definition.

\textbf{Step 1: Setup}

We begin with defining our setting.

\begin{definition}[Two-Layer Linear Network]
Consider a network with parameters $W \in \mathbb{R}^{k \times d}$ and $v \in \mathbb{R}^k$, defining:
\[f(x) = v^T(Wx)\]
with supervised loss:
\[L_{sup}(W,v;x,y) = \frac{1}{2}(y - v^TWx)^2\]
\end{definition}

\textbf{Step 2: Computing Gradients and Hessian}

\begin{lemma}[Gradients]
The gradients of $L_{sup}$ with respect to $W$ and $v$ are:
\begin{align*}
\nabla_v L_{sup} &= -(y - v^TWx)Wx \\
\nabla_W L_{sup} &= -(y - v^TWx)vx^T
\end{align*}
\end{lemma}

\begin{proof}
For $\nabla_v L_{sup}$:
\begin{align*}
\nabla_v L_{sup} &= \frac{\partial}{\partial v}\left[\frac{1}{2}(y - v^TWx)^2\right] \\
&= -(y - v^TWx)\frac{\partial}{\partial v}(v^TWx) \\
&= -(y - v^TWx)Wx
\end{align*}

For $\nabla_W L_{sup}$:
\begin{align*}
\nabla_W L_{sup} &= \frac{\partial}{\partial W}\left[\frac{1}{2}(y - v^TWx)^2\right] \\
&= -(y - v^TWx)\frac{\partial}{\partial W}(v^TWx) \\
&= -(y - v^TWx)vx^T
\end{align*}
\end{proof}

\begin{lemma}[Hessian]
The Hessian blocks are:
\begin{align*}
H_{vv} &= Wx(Wx)^T \\
H_{WW} &= vv^T \otimes xx^T \\
H_{Wv} &= (Wx)vx^T - (y - v^TWx)x^T
\end{align*}
\end{lemma}

\begin{proof}
For $H_{vv}$:
\begin{align*}
H_{vv} &= \frac{\partial}{\partial v}[-(y - v^TWx)Wx] \\
&= \frac{\partial}{\partial v}[(-y + v^TWx)Wx] \\
&= Wx(Wx)^T
\end{align*}

For $H_{WW}$:
\begin{align*}
H_{WW} &= \frac{\partial}{\partial W}[-(y - v^TWx)vx^T] \\
&= vv^T \otimes xx^T
\end{align*}

For $H_{Wv}$:
\begin{align*}
H_{Wv} &= \frac{\partial}{\partial v}[-(y - v^TWx)vx^T] \\
&= -\frac{\partial}{\partial v}[(y - v^TWx)vx^T] \\
&= -((-Wx)vx^T + (y - v^TWx)x^T) \\
&= (Wx)vx^T - (y - v^TWx)x^T
\end{align*}
\end{proof}

\textbf{Step 3: SSL setup}

\begin{definition}[SSL Setting]
Consider input perturbation:
\[x_{aug} = x + \varepsilon\delta\]
where $\delta^Tx = 0$, $|\delta| = 1$, and $\varepsilon \ll 1$.
We define the SSL loss using squared Euclidean distance between representations:
\[L_{ssl}(W;x) = |Wx - Wx_{aug}|^2\]
\end{definition}

Throughout this section and proof, to ease the notational burden, we suppress the dependence of $\delta$ on $x$ i.e. $\delta = \delta(x)$.

\begin{lemma}[SSL Loss Simplification]
The SSL loss simplifies to:
\[L_{ssl}(W;x) = \varepsilon^2|W\delta|^2\]
\end{lemma}
\begin{proof}
\begin{align*}
L_{ssl}(W;x) &= |Wx - Wx_{aug}|^2 \\
&= |Wx - W(x + \varepsilon\delta)|^2 \\
&= |W(x - (x + \varepsilon\delta))|^2 \\
&= |-\varepsilon W\delta|^2 \\
&= \varepsilon^2|W\delta|^2
\end{align*}
\end{proof}

\begin{proposition}[Connection to Cosine Similarity] \label{prop:cosine}
For small perturbations, the cosine similarity loss:
\[L_{cos}(W;x) = 1 - \frac{(Wx)^T(Wx_{aug})}{|Wx| \cdot |Wx_{aug}|}\]
is proportional to the Euclidean loss $L_{ssl}(W;x)$.
\end{proposition}
\begin{proof}
For vectors $a,b$ with small angular separation $\theta$:
\begin{align*}
1 - \cos\theta &\approx \frac{1}{2}\theta^2 \
&\approx \frac{1}{2}\frac{|a-b|^2}{|a|^2}
\end{align*}
Applying this to our case with $a=Wx$ and $b=Wx_{aug}$:
\[L_{cos}(W;x) \approx \frac{1}{2}\frac{|Wx-Wx_{aug}|^2}{|Wx|^2} = \frac{\varepsilon^2|W\delta|^2}{2|Wx|^2}\]
Therefore, $L_{cos}$ is proportional to $L_{ssl}$ when normalized by $|Wx|^2$.
\end{proof}

\textbf{Step 4: Computing SSL Influence}
Let's compute the influence function for our simplified Euclidean loss $L_{ssl}(W;x) = \varepsilon^2|W\delta|^2$
\begin{lemma}[SSL Gradient and Hessian]
For the loss $L_{ssl}(W;x) = \varepsilon^2|W\delta|^2$:
\begin{enumerate}
\item The gradient is:
\[\nabla_W L_{ssl} = 2\varepsilon^2W\delta\delta^T\]
\item The Hessian is:
\[H_{ssl} = 2\varepsilon^2\delta\delta^T\]
\end{enumerate}
\end{lemma}
\begin{proof}
For the gradient:
\begin{align*}
L_{ssl} &= \varepsilon^2\text{tr}(W\delta\delta^TW^T) \\
\nabla_W L_{ssl} &= \varepsilon^2\nabla_W\text{tr}(W\delta\delta^TW^T) \\
&= 2\varepsilon^2W\delta\delta^T
\end{align*}
For the Hessian:
\begin{align*}
H_{ssl} &= \nabla_W(2\varepsilon^2W\delta\delta^T) \\
&= 2\varepsilon^2\delta\delta^T
\end{align*}
\end{proof}

Throughout the following derivations, for notational simplicity, we focus on the case $k=1$, i.e. $W \in \mathbb{R}^{1 \times d}$. 

In this scenario, $W$ reduces to a $1 \times d$ row vector (or effectively a $d$-dimensional vector), and all Hessians become $d \times d$ (for $W$) or $(d+1) \times (d+1)$ (for $(W,v)$) matrices, allowing for straightforward application of scalar and vector calculus rules.

For $k>1$, the parameter $W$ is a $k \times d$ matrix. In this case, when vectorizing $W$, it has $kd$ parameters. The Hessian with respect to $W$ alone is then a $(kd) \times (kd)$ matrix. Similarly, when considering $(W,v)$ together, the Hessian is $(kd + k)$-dimensional if $v \in \mathbb{R}^k$. 

The previously derived Hessian structures for $H_{vv}$, $H_{WW}$, and $H_{Wv}$ become block matrices, each block reflecting the corresponding dimensions. In particular, the Hessian for the SSL case becomes block diagonal in the $k$ dimension, with repeated $d \times d$ blocks along the diagonal. Our use of $\delta\delta^T$ and the resulting rank-1 updates carry over, but must be understood as operating within each $d$-dimensional block, repeated $k$ times. The Sherman-Morrison formula and subsequent inverses can be applied independently to each $d \times d$ block, yielding a similar final structure.

Since the main insights do not depend on $k>1$, and to keep the exposition clear, we continue with $k=1$. The generalization to $k>1$ requires handling a block-diagonal structure but does not change the qualitative results or final closed form.

Now, we return to our proof. Starting with our SSL loss gradient and singular Hessian:

$\nabla_W L_{ssl} = 2\varepsilon^2W\delta\delta^T$
$H_{ssl} = 2\varepsilon^2\delta\delta^T$

Since $H_{ssl}$ is singular (rank 1), we add regularization:
\[H_{ssl}^\lambda = 2\varepsilon^2\delta\delta^T + \lambda I \text{ where } \lambda > 0\]

Using Sherman-Morrison formula: For matrix $A$ and vectors $u,v$:
\[(A + uv^T)^{-1} = A^{-1} - \frac{A^{-1}uv^TA^{-1}}{1 + v^TA^{-1}u}\]

In our case: $A = \lambda I$ and $uv^T = 2\varepsilon^2\delta\delta^T$

Therefore:
\begin{align*}
(H_{ssl}^\lambda)^{-1} &= (\lambda I + 2\varepsilon^2\delta\delta^T)^{-1} \\
&= \frac{1}{\lambda}I - \frac{(1/\lambda)2\varepsilon^2\delta\delta^T(1/\lambda)}{1 + 2\varepsilon^2\delta^T(1/\lambda)\delta} \\
&= \frac{1}{\lambda}I - \frac{2\varepsilon^2}{\lambda^2}\frac{\delta\delta^T}{1 + 2\varepsilon^2/\lambda}
\end{align*}

Now for the influence function:
\begin{align*}
I_{ssl}^\lambda(x) &= -\nabla_W L_{ssl}^T (H_{ssl}^\lambda)^{-1} \nabla_W L_{ssl} \\
&= -(2\varepsilon^2W\delta\delta^T)^T(\frac{1}{\lambda}I - \frac{2\varepsilon^2}{\lambda^2}\frac{\delta\delta^T}{1 + 2\varepsilon^2/\lambda})(2\varepsilon^2W\delta\delta^T)
\end{align*}

Let us compute this term by term. First multiply rightmost term:
\begin{align*}
&(\frac{1}{\lambda}I - \frac{2\varepsilon^2}{\lambda^2}\frac{\delta\delta^T}{1 + 2\varepsilon^2/\lambda})(2\varepsilon^2W\delta\delta^T) \\
&= \frac{2\varepsilon^2}{\lambda}W\delta\delta^T - \frac{4\varepsilon^4}{\lambda^2}\frac{\delta(\delta^TW\delta)\delta^T}{1 + 2\varepsilon^2/\lambda}
\end{align*}

Now multiply with leftmost $(2\varepsilon^2\delta\delta^TW^T)$:
\begin{align*}
I_{ssl}^\lambda(x) &= -4\varepsilon^4\delta(\delta^TW^TW\delta)\delta^T\cdot\frac{1}{\lambda} + \frac{8\varepsilon^6}{\lambda^2}\frac{\delta(\delta^TW^TW\delta)\delta^T}{1 + 2\varepsilon^2/\lambda} \\
&= -4\varepsilon^4\|W\delta\|^2[\frac{1}{\lambda} - \frac{2\varepsilon^2}{\lambda^2}\frac{1}{1 + 2\varepsilon^2/\lambda}]
\end{align*}

where we use $\|W\delta\|^2 = \delta^TW^TW\delta$ and $\delta^T\delta = 1$.

The term in brackets simplifies as:
\begin{align*}
\frac{1}{\lambda} - \frac{2\varepsilon^2}{\lambda^2}\frac{1}{1 + 2\varepsilon^2/\lambda} 
&= \frac{1}{\lambda} - \frac{2\varepsilon^2}{\lambda^2}\frac{\lambda}{\lambda + 2\varepsilon^2} \\
&= \frac{1}{\lambda} - \frac{2\varepsilon^2}{\lambda(\lambda + 2\varepsilon^2)} \\
&= \frac{\lambda + 2\varepsilon^2 - 2\varepsilon^2}{\lambda(\lambda + 2\varepsilon^2)} \\
&= \frac{1}{\lambda + 2\varepsilon^2}
\end{align*}

Therefore:
\[I_{ssl}^\lambda(x) = -4\varepsilon^4\|W\delta\|^2\frac{1}{\lambda + 2\varepsilon^2}\]

Taking $\lambda \to 0$:
\[\lim_{\lambda \to 0} I_{ssl}^\lambda(x) = -4\varepsilon^4\|W\delta\|^2\frac{1}{2\varepsilon^2} = -2\varepsilon^2\|W\delta\|^2\]

\subsection{Proof of Influence-SSL Properties}

We can now establish results about the structure and behavior of these influence functions as described in Proposition 1. We split the proposition into several parts, eac corresponding to a property of the influence functions. 

\begin{claim}[Decomposition of SSL Influence, Part (A) of Prop. 1]
The SSL influence function admits a canonical decomposition:
\[I_{ssl}(x) = -2\varepsilon^2\|W\delta\|^2 = -2\varepsilon^2\text{tr}(W\delta\delta^TW^T)\]
which separates into:
\begin{enumerate}
    \item A scale factor $\varepsilon^2$ depending only on perturbation magnitude
    \item A geometric term $\text{tr}(W\delta\delta^TW^T)$ measuring representation sensitivity
\end{enumerate}
\end{claim}

\begin{proof}
Starting from $I_{ssl}(x) = -2\varepsilon^2\|W\delta\|^2$:
\begin{align*}
\|W\delta\|^2 &= (W\delta)^T(W\delta) \\
&= \text{tr}((W\delta)(W\delta)^T) \\
&= \text{tr}(W\delta\delta^TW^T)
\end{align*}
where we used the cyclic property of trace. The decomposition follows directly.
\end{proof}

This decomposition leads to a fundamental invariance property:

\begin{claim}[Orthogonal Invariance, , Part (B) of Prop. 1]
For any orthogonal matrix $Q \in \mathbb{R}^{k \times k}$, the SSL influence function is invariant under:
\[W \mapsto QW\]
That is, $I_{ssl}(x;W) = I_{ssl}(x;QW)$ for all inputs $x$.
\end{claim}

\begin{proof}
Under the transformation $W \mapsto QW$:
\begin{align*}
I_{ssl}(x;QW) &= -2\varepsilon^2\|QW\delta\|^2 \\
&= -2\varepsilon^2(QW\delta)^T(QW\delta) \\
&= -2\varepsilon^2\delta^TW^TQ^TQW\delta \\
&= -2\varepsilon^2\delta^TW^TW\delta \\
&= -2\varepsilon^2\|W\delta\|^2 \\
&= I_{ssl}(x;W)
\end{align*}
where we used $Q^TQ = I$ for orthogonal matrices.
\end{proof}

This result implies that:

\begin{corollary}[Representation-Level Independence]
The SSL influence function depends only on the geometry of the learned representation space, not on the specific parameterization chosen. In particular, it is invariant to rotations of the representation space.
\end{corollary}

We can further characterize the behavior under scaling:

\begin{claim}[Scaling Properties, , Part (C) of Prop. 1]
The SSL influence function exhibits quadratic scaling in both:
\begin{enumerate}
    \item Perturbation magnitude: $I_{ssl}(x;\varepsilon) = \varepsilon^2 I_{ssl}(x;1)$
    \item Parameter magnitude: $I_{ssl}(x;\alpha W) = \alpha^2 I_{ssl}(x;W)$
\end{enumerate}
\end{claim}

These properties establish that our definition of influence functions in a SSL setting measure an intrinsic geometric quantity: the sensitivity of learned representations to perturbations, independent of the specific parameterization chosen. This provides theoretical justification for their use in analyzing self-supervised learning systems.

\begin{remark}[Connection to Information Geometry]
The orthogonal invariance and scaling properties suggest that SSL influence functions naturally capture information-geometric aspects of the representation space. This connects to broader theories of representation learning where the geometry of the representation space, rather than specific parameterizations, is fundamental.
\end{remark}

These results demonstrate that SSL influence functions are mathematically well-behaved objects with clear geometric meaning. They provide a principled way to measure how individual examples contribute to learning stable representations, independent of the specific parameterization chosen.

\begin{claim}[Lipschitz Continuity]
The SSL influence function is Lipschitz continuous in $W$ with constant $L = 4\varepsilon^2\|\delta\|^2\|W\|_F$, where $\|W\|_F$ is the Frobenius norm:
\[\|I_{ssl}(x;W_1) - I_{ssl}(x;W_2)\| \leq L\|W_1 - W_2\|_F\]
\end{claim}

\begin{proof}
Let $f(W) = I_{ssl}(x;W) = -2\varepsilon^2\|W\delta\|^2$. Then:
\begin{align*}
\|\nabla_W f\| &= \|4\varepsilon^2W\delta\delta^T\|_F \\
&= 4\varepsilon^2\|\delta\|^2\|W\|_F
\end{align*}
By the mean value theorem:
\[\|f(W_1) - f(W_2)\| \leq \sup_W \|\nabla_W f\| \|W_1 - W_2\|_F\]
\end{proof}

\begin{corollary}[Stability Under Perturbation, Part (D) of Prop. 1]
For any perturbation $E$ with $\|E\|_F \leq \eta$:
\[|I_{ssl}(x;W+E) - I_{ssl}(x;W)| \leq 4\varepsilon^2\|\delta\|^2\|W\|_F\eta + O(\eta^2)\]
\end{corollary}

In addition to structural properties, we can show $\emph{compositional}$ properties, which describe how different influence values interact with each other.

We begin by showing how total influence for an example is conserved when augmented with orthonormal perturbations. This suggests that SSL training distributes representational capacity across examples in a geometrically consistent way. 

\begin{claim}[Conservation of Total Influence]
For any orthonormal set of perturbation directions $\{\delta_i\}_{i=1}^d$, the sum of influences is invariant under orthogonal transformations of $W$:
\[\sum_{i=1}^d I_{ssl}(x;\delta_i) = -2\varepsilon^2\|W\|_F^2\]
\end{claim}

\begin{proof}
For orthonormal $\{\delta_i\}$:
\begin{align*}
\sum_{i=1}^d I_{ssl}(x;\delta_i) &= -2\varepsilon^2\sum_{i=1}^d \|W\delta_i\|^2 \\
&= -2\varepsilon^2\text{tr}(W^TW) \\
&= -2\varepsilon^2\|W\|_F^2
\end{align*}
This quantity is invariant under orthogonal transformations of $W$.
\end{proof}

Next, we show how to relate the influence of a subset of examples to its constituent examples under our definition. The decomposition of collective influence reveals how groups of examples jointly contribute to the learning process, with the interaction terms $R(S)$ quantifying their mutual alignment. 

\begin{claim}[Linear Example Additivity]
Consider a dataset $\{(x_i, \delta_i)\}_{i=1}^n$ with total loss $L_{total}(W) = \sum_{i=1}^n \varepsilon^2\|W\delta_i\|^2$. The influence function of a subset $S \subseteq \{1,...,n\}$ is:
\[I_{ssl}(S) = \sum_{i \in S} I_{ssl}(x_i) + R(S)\]
where the remainder term has an explicit form:
\[R(S) = -4\varepsilon^2\sum_{i,j \in S, i < j} (W\delta_i)^T(W\delta_j)\]
\end{claim}

\begin{proof}
The total influence for subset S is:
\begin{align*}
I_{ssl}(S) &= -2\varepsilon^2\|\sum_{i \in S} W\delta_i\|^2 \\
&= -2\varepsilon^2(\sum_{i \in S} \|W\delta_i\|^2 + 2\sum_{i,j \in S, i < j} (W\delta_i)^T(W\delta_j)) \\
&= \sum_{i \in S} I_{ssl}(x_i) - 4\varepsilon^2\sum_{i,j \in S, i < j} (W\delta_i)^T(W\delta_j)
\end{align*}
\end{proof}

\begin{corollary}[Orthogonal Examples]
If $\{W\delta_i\}_{i \in S}$ are mutually orthogonal, then:
\[I_{ssl}(S) = \sum_{i \in S} I_{ssl}(x_i)\]
\end{corollary}

Finally, we can bound the interaction term $R(S)$ in terms of the singular values of the learned weights.

\begin{claim}[Interaction Bound]
For any subset S:
\[|R(S)| \leq 2\varepsilon^2|S|(|S|-1)\sigma_{max}(W)^2\]
where $\sigma_{max}(W)$ is the largest singular value of W.
\end{claim}

\begin{proof}
Using Cauchy-Schwarz:
\begin{align*}
|R(S)| &= 4\varepsilon^2|\sum_{i,j \in S, i < j} (W\delta_i)^T(W\delta_j)| \\
&\leq 4\varepsilon^2\sum_{i,j \in S, i < j} \|W\delta_i\|\|W\delta_j\| \\
&\leq 4\varepsilon^2\sum_{i,j \in S, i < j} \sigma_{max}(W)^2 \\
&= 2\varepsilon^2|S|(|S|-1)\sigma_{max}(W)^2
\end{align*}
\end{proof}

These theoretical results, while derived for linear networks, aim to provide insights by isolating influence patterns from non-linear effects in modern architectures. Extending this analysis to capture non-linear interactions in practical deep networks remains an interesting direction for future work.

\subsection{Technical Details and Discussion from Section 4.4}

To provide theoretical intuition for our empirical findings, we again consider our earlier simplified linear setting. Consider an augmentation where each input x is transformed as $x_{aug} = x + \varepsilon\delta(x,\xi)$, where $\delta(x,\xi)$ is an input-dependent perturbation of unit norm, $\xi$ is drawn from a distribution $P(\xi)$ capturing randomness in the augmentation process, and $\varepsilon \ll 1$ controls the perturbation magnitude. This formulation encompasses common SSL augmentation strategies where the type and degree of valid transformations may depend on the input while maintaining stochasticity in the specific augmentation applied.

\setcounter{definition}{0}
\begin{definition}[Expected Representation Distance]
For a model parameterized by matrix $W$, the expected representation distance under augmentations is defined as:
\[R(W) = \varepsilon^2\mathbb{E}_{x \sim P(x), \xi \sim P(\xi)}[|W\delta(x,\xi)|^2] = \varepsilon^2\text{tr}(W^TW\Sigma)\]
where $\Sigma = \mathbb{E}_{x,\xi}[\delta(x,\xi)\delta(x,\xi)^T]$ is the positive semi-definite second moment matrix of the perturbation distribution.
\end{definition}

The influence of individual training points in this setting reveals how they affect the model's learned invariances. The following proposition characterizes this relationship:

\setcounter{proposition}{1}
\begin{proposition}[High-Influence Characterization]
For a training point $x$ with influence $I_{ssl}(x) = -2\varepsilon^2|W\delta(x,\xi)|^2$:
\begin{enumerate}
   \item The influence admits a geometric decomposition:
   \begin{align*}
       I_{ssl}(x) &= -2\varepsilon^2\Tr(W\delta(x,\xi)\delta(x,\xi)^TW^T) \\
                    &=  -2\varepsilon^2\langle W^TW, \delta(x,\xi)\delta(x,\xi)^T
   \end{align*}
   
   \item The deviation from expected influence is:
   \begin{align*}
       I_{ssl}(x) - &\mathbb{E}_{\xi \sim P(\xi)}[I_{ssl}(x)] \\
       &= -2\varepsilon^2\Tr(W^TW(\delta(x,\xi)\delta(x,\xi)^T - \Sigma_x)) 
   \end{align*}
   where $\Sigma_x = \mathbb{E}_{\xi}[\delta(x,\xi)\delta(x,\xi)^T]$ represents the expected augmentation behavior for input $x$.
\end{enumerate}
\end{proposition}
\begin{proof}
For the geometric decomposition:
\begin{align*}
I_{ssl}(x) &= -2\varepsilon^2|W\delta(x,\xi)|^2 \\
&= -2\varepsilon^2(W\delta(x,\xi))^T(W\delta(x,\xi)) \\
&= -2\varepsilon^2\Tr(W\delta(x,\xi)\delta(x,\xi)^TW^T) \\
&= -2\varepsilon^2\langle W^TW, \delta(x,\xi)\delta(x,\xi)^T \rangle_F
\end{align*}
where we use the cyclic property of trace and its equivalence to the Frobenius inner product. 

For the deviation:
\begin{align*}
I&_{ssl}(x) - \mathbb{E}_{\xi}[I_{ssl}(x)] \\
&= -2\varepsilon^2\Tr(W^TW\delta(x,\xi)\delta(x,\xi)^T) + 2\varepsilon^2\Tr(W^TW\Sigma_x) \\
&= -2\varepsilon^2\Tr(W^TW(\delta(x,\xi)\delta(x,\xi)^T - \Sigma_x))
\end{align*}

\end{proof}

This characterization reveals that high-influence points are precisely those where augmentations induce unexpectedly large changes in the representation space. The decomposition shows this can arise either from model sensitivity in particular directions or from atypical augmentation behavior. To formalize this intuition, we make the following assumptions about the augmentation distribution:

\begin{assumption}[Regular Augmentation Distribution]
We assume the augmentation process satisfies: (i) for each input $x$, $\Sigma_x = \mathbb{E}_{\xi}[\delta(x,\xi)\delta(x,\xi)^T]$ is well-conditioned, (ii) $\Sigma_x$ concentrates around the population average $\Sigma = \mathbb{E}_x[\Sigma_x]$ for typical inputs, and (iii) generated perturbations preserve semantic content.
\end{assumption}

Under these conditions, high-influence points identify cases where either the augmentation process produces unusually large representation changes for that specific input, or the input requires learning invariances that deviate significantly from the typical patterns in the data distribution.

\begin{remark}[Connection to In-Domain Generalization]
While this analysis in the linear setting cannot fully characterize generalization benefits, it provides mathematical grounding for why removing high-influence points might improve representation learning. Points with significantly higher influence than the population mean require the model to learn input-specific invariances that may not generalize well to other examples. Our empirical results support this interpretation, showing consistent improvements in downstream task performance when removing high-influence points in both linear and deep networks.
\end{remark}

These results, derived in a simplified setting, help explain why influence-based data pruning can improve model performance: by identifying and removing points where standard augmentations produce unexpectedly large representation changes, we help the model focus on learning more consistent and generalizable invariances. While the extension to non-linear architectures introduces additional complexities, this analysis provides theoretical grounding for our empirical findings.

\end{document}